\documentclass{amsart}
\usepackage{graphicx} 
\usepackage[utf8]{inputenc} 
\usepackage[T1]{fontenc}    
\usepackage{hyperref}       
\usepackage{url}            
\usepackage{booktabs}       
\usepackage{amsfonts}       
\usepackage{nicefrac}       
\usepackage{microtype}      
\usepackage{xcolor}         
\usepackage{latexsym}
\usepackage{epsfig}
\usepackage{amssymb}
\usepackage{amstext}
\usepackage{amsgen}
\usepackage{amsxtra}
\usepackage{amsgen}
\usepackage{amsthm,mathrsfs,mathtools}
\usepackage{bbm}
\usepackage{enumerate,amsmath,subfigure,color}
\usepackage{tikz}
\usepackage{comment}
\usepackage{ulem}
\usepackage{enumitem}
\usepackage{siunitx}
\usepackage[foot]{amsaddr}

\makeatletter
\newcommand{\addresseshere}{%
  \enddoc@text\let\enddoc@text\relax
}
\makeatother

\newtheorem{thm}{Theorem}[section]
\newtheorem{proposition}[thm]{Proposition}
\newtheorem{lemma}[thm]{Lemma}

\theoremstyle{remark}
\newtheorem{rem}[thm]{Remark}

\theoremstyle{definition}

\newtheorem{definition}[thm]{Definition}

\newtheorem{assumption}[thm]{Assumption}

\newcommand{\N}{\mathbb{N}}

\newcommand{\R}{\mathbb{R}}

\newcommand{\E}{\mathbb{E}}

\renewcommand{\epsilon}{\varepsilon} 

\newcommand{\prox}{\ensuremath{\mathrm{prox}}} 
\newcommand{\sign}{\ensuremath{\mathrm{sign}}} 
\newcommand{\var}{\ensuremath{\mathrm{Var}}}
\newcommand{\cov}{\ensuremath{\mathrm{Cov}}}
\newcommand{\softmax}{\ensuremath{\mathrm{softmax}}}

\DeclareMathOperator*{\argmin}{arg\,min}

\def\<{\langle}
\def\>{\rangle}

\title[Learning a Gaussian Mixture for Sparse Optimization]{Learning a Gaussian Mixture \\ for Sparsity Regularization in Inverse Problems}

\author[G.\ S.\ Alberti]{Giovanni S.\ Alberti \!$^1$} \email{giovanni.alberti@unige.it}
\author[L.\ Ratti]{Luca Ratti \!$^2$} \email{luca.ratti5@unibo.it}
\author[M.\ Santacesaria]{Matteo Santacesaria \!$^1$} \email{matteo.santacesaria@unige.it}
\author[S.\ Sciutto]{Silvia Sciutto \!$^1$} \email{silvia.sciutto@edu.unige.it}
\address{$^1$ MaLGa Center, Department of Mathematics, University of Genoa, Via Dodecaneso 35, 16146 Genova, Italy.}
\address{$^2$ Department of Mathematics, University of Bologna, Piazza di Porta San Donato 5, 40126 Bologna, Italy.}

\keywords{Sparse optimization, statistical learning, inverse problems, Gaussian mixture, dictionary learning, Bayes estimator}

\begin{document}

\begin{abstract}
In inverse problems, it is widely recognized that the incorporation of a sparsity prior yields a regularization effect on the solution. This approach is grounded on the a priori assumption that the unknown can be appropriately represented in a basis with a limited number of significant components, while most coefficients are close to zero. This occurrence is frequently observed in real-world scenarios, such as with piecewise smooth signals.

In this study, we propose a probabilistic sparsity prior formulated as a mixture of degenerate Gaussians, capable of modeling sparsity with respect to a generic basis. Under this premise, we design a neural network that can be interpreted as the Bayes estimator for linear inverse problems. Additionally, we put forth both a supervised and an unsupervised training strategy to estimate the parameters of this network.

To evaluate the effectiveness of our approach, we conduct a numerical comparison with commonly employed sparsity-promoting regularization techniques, namely LASSO, group LASSO, iterative hard thresholding, and sparse coding/dictionary learning. Notably, our reconstructions consistently exhibit lower mean square error values across all $1$D datasets utilized for the comparisons, even in cases where the datasets significantly deviate from a Gaussian mixture model.
    
\end{abstract}

\maketitle

\section{Introduction}

A signal is said to be \textit{sparse} if it can be represented as a linear combination of a small number of vectors in a known family (e.g., a basis or a frame). Sparsity has played a major role in the last decades in signal processing and statistics as a way to identify key quantities and find low dimensional representations of many families of signals, including natural images \cite{candes2006robust,donoho2006compressed,mallat2008wavelet}. In inverse problems, sparse optimization has had a significant impact in many applications, notably in accelerated magnetic resonance imaging via compressed sensing \cite{lustig2008compressed}.

The key ingredient needed in sparse optimization is the knowledge of a suitable dictionary that can well represent the unknown quantities with as few elements as possible. If these quantities cannot be described analytically and are measured from experiments, given a sufficient number of samples, machine learning techniques can be used to infer the dictionary {\cite{lee2006efficient,horesh-haber-2009,huang-haber-horesh-2012}}. Dictionary learning in the context of sparse optimization, i.e. sparse coding, can be seen as the problem of finding the best change of basis that makes the data as sparse as possible. Thus, it is a special case of the problem of learning a regularization functional \cite{lunz2018adversarial,schwab2019deep,alberti2021learning} and the more general framework of learning operators between function spaces (see {\cite{de2023convergence,kovachki2023neural}} and references therein).

In this work, we propose a new approach for dictionary learning for sparse optimization motivated by the problem of learning a regularization functional for solving inverse problems. We consider a linear inverse problem in a finite-dimensional setting:
\begin{equation}
\label{IP}
y = A x + \epsilon,   
\end{equation}
where $x \in \R^n$, $\epsilon \in \R^m$, $A \in \R^{m \times n}.$ 
We assume a statistical perspective: namely, $x$ and $\epsilon$ are realizations of the random variables $X$ and $E$ on $\R^n$ and $\R^m$, respectively. Let $\rho_X$ and $\rho_E$ denote their probability distributions.

Given some partial knowledge about the probability distributions of $X$ and $E$, we aim to recover an estimator $R\colon\R^m\to\R^n$ that has statistical guarantees and encodes information on the sparsity of $X$. While sparsity is well established in a deterministic setting, there is no clear consensus on how to define a probability distribution of sparse vectors. One option comes from the Bayesian interpretation of $\ell^1$ minimization: it can be seen as a maximum a posteriori (MAP) estimator under a Laplacian prior in the presence of additive white Gaussian noise. However, this is unsatisfactory if the aim is to generate truly sparse signals. Recently, it has been shown that specific hierarchical Bayesian models can be used as priors for sparse signals \cite{2019-calvetti-etal,Calvetti2023}.

Instead, in this work we consider a mixture of degenerate Gaussians as a statistical model for a distribution of sparse vectors $X$. The dimensions of the degenerate support of each Gaussian naturally represent the sparsity levels of the signals. We also assume that the noise is Gaussian.

To find the estimator $R$, we consider a statistical learning framework. We employ the mean squared error (MSE, also referred to as expected risk), namely:
\begin{equation}
\label{expected_risk}
L(R) =\E_{X \sim \rho_X,E\sim \rho_E}\| R(AX+E)-X \|_{\R^m}^2 = \E_{(X,Y) \sim \rho}\| R(Y)-X \|_{\R^m}^2,    
\end{equation}
where $\rho$ is the joint probability distribution of $(X,Y)$ on $\R^n \times \R^m$. By minimizing $L$ among all possible measurable functions $R$, we get the so-called Minimum MSE (MMSE) estimator, or Bayes estimator.

The Bayes estimator cannot be directly employed as a solution to the statistical inverse problem because it requires the knowledge of the full probability distribution $\rho$. Supervised learning techniques can be used to find a good approximation when a finite amount of data is available. However, a suitable class of functions (hypothesis space) must be chosen to obtain a good approximation.

Our main contribution is the development of supervised and unsupervised learning schemes to approximate the Bayes estimator {when a (degenerate) Gaussian Mixture model is assumed as a prior on $X$}. This is done thanks to an explicit expression for the MMSE, which is usually not available for general probability distributions underlying the data and the noise. The proposed training schemes are based on the observation that the explicit expression of the MMSE can be seen as a two-layer neural network, which can be easily trained using stochastic optimization and back-propagation. In particular, we find a strong similarity of our network with a single self-attention layer, a building block of the well-known transformer architecture \cite{vaswani2017attention}: our results provide a novel statistical interpretation of the attention mechanism.

{As carefully presented in Section \ref{sec:our_alg}, the supervised technique performs an empirical risk minimization within the hypothesis class $\mathcal{H}$ of Bayes estimators of Gaussian Mixture models, the means and the covariance matrices of the mixture being learnable parameters. The unsupervised one, instead, approximates the Bayes estimator relying on suitable clustering techniques and empirical estimates for the means and covariances.}

After the training, our network/algorithm can be used as an alternative to classical sparse optimization approaches. We perform extensive numerical comparisons of our approach with well-known sparsity-promoting algorithms: LASSO, group LASSO, iterative hard thresholding (IHT), and sparse coding/dictionary learning. The tests are done by performing denoising and deblurring on several {datasets of (discretized) 1D signals}. In all experiments, our method outperforms the others, in the MSE. Our method can be seen as a new paradigm for both sparse optimization and dictionary learning for sparse data.

The main limitation of our strategy is represented by the number of parameters that need to be estimated to identify the optimal regularizer. It is easy to show that this number grows as $L n^2$, being $L$ the number of components of the Gaussian mixture model. If the unknown $X$ is simply assumed to be $s$-sparse, namely, that its support is contained in any possible $s-$dimensional coordinate subspace in $\R^n$, we should consider $L=\binom{n}{s}$, which is prohibitive in most applications. This is why our method is best suited for cases of \textit{structured} sparsity, in which the support of $X$ consists of a restricted number of subspaces, which is the case of group sparsity \cite{yuan2006model}, for instance.

The paper is organized as follows. In Section~\ref{sec:stat_learn_bayes_est}, we provide an overview of the concept of Bayes estimator and recall an explicit formula in the case of linear inverse problems with Gaussian mixture prior. Afterwards, Section~\ref{sec:network_bayes} is dedicated to the interpretation of this estimator as a neural network, highlighting its potential utility for sparse recovery. We introduce a supervised and an unsupervised 
{approach} for learning the neural network representation of the Bayes estimator in Section~\ref{sec:our_alg}. In Section~\ref{sec:comp}, we describe some baseline sparsity-promoting algorithms, against which our methods are evaluated. This is done in Section~\ref{sec:numerics}, where we conduct numerical comparisons on denoising and deblurring problems across various {datasets of 1D signals}.
Conclusions are drawn in Section~\ref{sec:conclusions}. Further contents are reported in the appendices, including the proof of the main theoretical result (Appendix \ref{proof_bayes_est}), implementation details about the baseline algorithms (Appendix \ref{app:algorithms}), and extended numerical comparisons (Appendix \ref{app:numerics}).

\section{Statistical learning and Bayes estimator for inverse problems}
\label{sec:stat_learn_bayes_est}

In this section, we introduce our approach and motivation, collecting some results that are already known about inverse problems, statistical estimation, and mixtures of Gaussian random variables. The main elements of novelty are reported in the next section.

We recall the inverse problem introduced in \eqref{IP}, which we can also interpret as a realization of the following equation involving the random variables $X$ on $\R^n$ and $E,Y$ on $\R^m$:
\begin{equation}
\label{statIP}
Y = AX + E.   
\end{equation}
Throughout the paper, We make the following assumptions on the random variables $X$ and $E$.
\begin{assumption}
We assume that:
\begin{itemize}
\item The distribution of the noise, $\rho_E$, is known and zero-mean, i.e.\ $\mathbb{E}[E] = \mathbf{0}$. The covariance $\Sigma_E$ is invertible.
    \item The random variables $X$ and $E$ are independent.
\end{itemize}
    \label{ass:0}
\end{assumption}

Our goal is to identify a function $R\colon \R^m \rightarrow \R^n$ (which we will refer to as a regularizer or an estimator) such that the reconstructions $R(Ax+\epsilon)$ are close to the corresponding $x$, when $x,\epsilon$ are sampled from $X,E$ and the squared error $\|R(Ax+\epsilon)-R(x)\|^2$ is used as a metric.

This problem is in general different from the deterministic formulation of inverse problems: the recovery of a single $x$ from $y=Ax+\epsilon$. The task we are considering also differs from the one of Bayesian inverse problems. In that scenario, indeed, the goal is not to retrieve a function $R$ (or a regularized solution $R(Y)$), but rather to determine a probability distribution, in particular the conditional distribution of $X|Y$, and to ensure that, as the noise $E$ converges to $\mathbf{0}$ in some suitable sense, this (posterior) distribution converges to the distribution of $X$.

Our perspective on the inverse problem is instead rather related to the statistical theory of estimation, or statistical inference. Indeed, based on some partial knowledge about the probability distributions of $X$ and $E$, we aim to recover an estimate $R(Y)$ of $X$ close to $X$ w.r.t.\ the MSE. In particular, we set our discussion at the intersection of two areas of statistical inference, namely the Minimum Mean Square Error estimation and parametric estimation.
 
The first field is determined by the choice of \eqref{expected_risk} as the metric according to which the random variables $R(Y)$ and $X$ are considered close to each other or not.
The minimizer of $L$ among all possible measurable functions $R$ is defined as the MMSE estimator, or Bayes estimator:
\begin{equation}
R^{\star} \in \argmin \{ L(R): \; R\colon \R^m \rightarrow \R^n, R \text{ measurable} \}.
\label{eq:Bayesian_est}    
\end{equation}
As it is easy to show (see \cite{cucker2002mathematical}), the solution to such a problem is the conditional mean of $X$ given $Y$, so that
\[
R^{\star}(y) = \E[X|Y = y] = \int_{\R^n} x \hspace{0.05cm} p(x|y) \hspace{0.05cm} dx.   
\]

The Bayes estimator $R^\star$ is not always the most straightforward choice. Indeed, its exact computation requires the knowledge of the joint probability distribution $\rho$, which depends on the distributions of the noise $E$ and of $X$. Notice that the distribution of $X$ is not a user-crafted prior but should encode the full statistical model of the ground truth $X$, which may not be fully accessible when solving the inverse problem.

To overcome this issue, one usually assumes to have access to a rather large set of training data, namely of pairs $(x_j,y_j)$ sampled from the joint distribution $\rho$. It is then possible to substitute the integrals appearing in the definition of $L$ and of $R^\star$ by means of Monte-Carlo quadrature rules. This approach is usually referred to as \textit{supervised statistical learning}.
In order to avoid overfitting and preserve stability, one typically restricts the choice of the possible estimators $R$ to a specific class of functions $\mathcal{H}$, which introduces an implicit bias on the estimator and possibly a regularizing effect. This leads us to the framework of parametric estimation.

In \cite{alberti2021learning}, for example, the hypothesis class $\mathcal{H}$ consists of a family of affine functions in $y$, solutions of suitably parametrized quadratic minimization problems. Minimizing over such a class (which corresponds to the task of learning the optimal generalized Tikhonov regularizer) is motivated by theoretical reasons: if the distributions of $X$ and $E$ are Gaussian, then $R^\star$ belongs to $\mathcal{H}$. However, in general the minimizer of $L$ within $\mathcal{H}$ would deviate from $R^\star$, and such a bias could be in some cases undesired.   On top of that, a significant outcome of the analysis in \cite{alberti2021learning} is a closed-form expression of the optimal regularizer (also in infinite dimension), which in particular can be computed if the mean and the covariance of $X$ are known. This paves the way to an \textit{unsupervised statistical learning} approach: namely, to approximate such an estimator one would not need a training set ${(x_j,y_j)}_j$ sampled from the probability distribution $\rho$, but only a set $\{x_j\}$ sampled from $\rho_X$. 

This raises the question if there exist other possible prior distributions $X$ that may lead to a simple choice of a hypothesis class $\mathcal{H}$ containing the corresponding Bayes estimators, possibly endowed with an unsupervised technique to approximate them. If we consider a variational approach, in which $\mathcal{H}$ is a set of solution maps of suitable minimization problems, \cite{gribonval2011should} showed that the conditional mean $R^\star$, i.e.\ the MMSE estimator, can always be represented as the minimizer of a functional $\Phi_{\text{MMSE}}$, related to the probability distribution $p_X$. Nevertheless, in most cases, such a functional is not convex, which makes it critical to define $\mathcal{H}$ and to solve a minimization problem in it. 

In this paper, without considering a variational point of view, we propose a strategy through which it is possible to associate a hypothesis class to (rather general) prior distributions of $X$. The most prominent outcome of this approach is the thorough treatment of the case in which $X$ is a mixture of Gaussian random variables, which may be employed as a model of sparsity.

Our approach is based on the idea of parametric estimation in statistical inference: the distribution of the exact solution $X$ is unknown, but belongs to a known class of distributions, parametrized by a suitable set of parameters: 
\[
\{\rho_X(\theta): \  \theta \in \Theta\},
\]
for some $\Theta \subseteq \R^p$.
For example, we can assume that $X$ is a Gaussian random variable of unknown mean $\mu$ and covariance $\Sigma$, the pair $(\mu,\Sigma)$ being the parameter $\theta$. Analogously, we can describe more complicated probability distributions, possibly employing a larger set of parameters.

In this context, it is 
{natural} to construct a hypothesis class $\mathcal{H}$ that contains the Bayes estimators $R^\star$ of every possible prior in the parametric class, namely,
\[
\mathcal H = \{ R_\theta=\mathbb{E}_{\rho(\theta)}[X|Y=\cdot]: \theta \in \Theta\},
\]
where $\rho(\theta)$ is the joint probability distribution for $(X,Y)$ when $X\sim \rho_X(\theta)$, $Y=AX+E$ and $E\sim\rho_E$. Recall that, by Assumption~\ref{ass:0}, the distribution of the noise $E$ is known, and so it is considered fixed. {However, the choice of this hypothesis class is non-standard, if compared to deep learning approaches exploiting the approximation power of neural networks. The two main motivations for its adoption are that the Bayes estimator has sharp statistical properties in terms of MSE, and moreover it is well-suited for both supervised and unsupervised learning approaches, as we explain below.}

Notice that this same strategy can also be employed for nonlinear inverse problems: we formulated it in this narrower context for the ease of notation, since the content of the next sections is only related to linear problems.

Despite this approach being rather general, it is useful only if it is possible to provide a closed-form expression of $R_\theta$, which might entail a learning approach for its approximation. This can be easily done when $X$ belongs to fairly simple classes, such as Gaussian random variables or, as we show in this section, mixtures of Gaussians. {Other possible examples entailing closed-form representation of the posterior distribution (thus, of the Bayes estimator) are represented by the families of conjugate priors associated with the specified likelihood, i.e., with the noise model: for example, in the case of a Poisson likelihood, if the prior has a Gamma distribution, then also the posterior does.}



\subsection{Bayes estimator for linear inverse problems with a Gaussian Mixture prior}
\label{subsec:Bayesian}

We first recall some basic definitions on mixtures of random variables. Then, in Theorem~\ref{thm:bayes_est} we show the formula of the Bayes estimator for linear inverse problems with a Gaussian mixture prior. 

\begin{definition}
\label{mix_var}
    A random variable $X$ in $\R^n$ is a \textit{mixture of random variables} if it can be written as
    \begin{equation*}
        X = \sum_{i=1}^L X_i \mathbbm{1}_{\{i\}}(I) ,
    \end{equation*}
    where $X_i$ are random variables in $\R^n$, $I$ is a random variable on $\{ 1,...,L \}$ independent of $X_i$ and $L \in \N^+$ is the number of elements in the mixture. The indicator function $\mathbbm{1}_{\{i\}}(I)$ is equal to $1$ when $I=i$ and $0$ otherwise: as a consequence, the role of the discrete random variable $I$ is selecting the random variable $X_i$, therefore $w_i := \mathbb{P}(I = i)$ are informally called the \textit{weights of the mixture}.
\end{definition}

\begin{definition}
\label{Gauss_mix_var}
    A random variable $X$ in $\R^n$ is a \textit{Gaussian mixture} if it is a mixture, as defined in Definition~\ref{mix_var}, and $X_i \sim \mathcal{N}(\mu_i,\Sigma_i)$ are Gaussian random variables in $\R^n$.
\end{definition}

\begin{thm}
\label{thm:bayes_est}
Let $X$ be a Gaussian mixture in $\R^n$, as in Definition~ \ref{Gauss_mix_var}, and let $E\sim \mathcal{N}(\mathbf{0},\Sigma_E)$ {be such that Assumption \ref{ass:0} is verified, i.e., $\Sigma_E$ is invertible and $E$ is independent of $X_i$ for every $i$, and of $I$}. 
Set $Y=A X+E$, as in  \eqref{statIP}. The corresponding Bayes estimator is
    \begin{equation}
    \label{Bayes_est_gauss_mix}
    R^{\star}(y) = \E[X|Y=y] = \sum_{i=1}^L \frac{c_i}{\sum_{j=1}^L c_j} (\mu_i + \Sigma_i A^T (A \Sigma_i A^T + \Sigma_E)^{-1} (y-A \mu_i)),  
    \end{equation}
    where
    \begin{equation}
    \label{c_i}
        c_i = \frac{w_i}{\sqrt{(2\pi)^n |A \Sigma_i A^T + \Sigma_E|}} \exp{ \Big( -\frac{1}{2} \| (A \Sigma_i A^T + \Sigma_E)^{-\frac{1}{2}} (y-A \mu_i) \|_2^2 \Big) },
    \end{equation}
    with the notation $|B|=\det B$.
\end{thm}
{
Note that the forward map $A$ and the covariance of the noise $\Sigma_E$ are considered known and fixed, and we assume that $\Sigma_E$ is invertible, so that $A \Sigma_i A^T + \Sigma_E$ is invertible for every $i$.}
A proof of Theorem~\ref{thm:bayes_est} can be found in \cite{kundu2008gmm}. For completeness, we provide a similar proof in Appendix~\ref{proof_bayes_est}.

\section{A neural network for sparse recovery}
\label{sec:network_bayes}

In this section, we provide a novel interpretation both of the  resulting formula \eqref{Bayes_est_gauss_mix} and of the Gaussian mixture model itself, in view of an application to sparsity-promoting learned regularization. We start by showing that the expression derived in \eqref{Bayes_est_gauss_mix} can be understood as a neural network, whose architecture has strong connections with the well-known attention mechanism used in transformers \cite{vaswani2017attention}. Then, we describe how the Gaussian mixture model assumption can be used to encode a (group) sparsity prior on the unknown random variable $X$.

\subsection{Bayes estimator as a neural network}
\label{sec:bayes_as_NN}

We wish to interpret the expression of the Bayes estimator \eqref{Bayes_est_gauss_mix} for Gaussian mixture models as a neural network from $\R^m$ to $\R^n$. 

We start by identifying the parameters: in particular, we define 
\begin{equation}
\theta = \big( \{w_i\}_{i=1}^L, \{\mu_i\}_{i=1}^L, \{\Sigma_i\}_{i=1}^L\big),
    \label{eq:theta}
\end{equation}
collecting all the weights, means, and covariances of the mixture. We can now denote the function defined in \eqref{Bayes_est_gauss_mix} as $R_\theta$, parametrized according to \eqref{eq:theta}. 
The resulting hypothesis class is 
\begin{equation}
\mathcal{H} = \{R_\theta: \ \theta \in \Theta\}; \quad \Theta \subseteq [0,1]^{L} \times (\R^n)^L \times (\R^{n\times n})^L.
    \label{eq:hypothesis}
\end{equation}
In order for $R_{\theta}$ to be well defined and for the parameters to have a precise statistical interpretation, we impose that 
\begin{equation*}
   \Theta = \{ \theta \text{ defined in \eqref{eq:theta}}: \sum_{i=1}^L w_i = 1 \text{ and } \Sigma_i \succcurlyeq 0 \text{ symmetric for } i=1,...,L\},
\end{equation*}
where $\Sigma_i \succcurlyeq 0$ denotes that $\Sigma_i$ is  positive semidefinite.
Notice that the set $\mathcal{H}$ is the collection of the Bayes estimators of all possible Gaussian mixture models of $L$ components in $\R^n$.

We now focus on describing formula \eqref{Bayes_est_gauss_mix} in terms of a neural network's architecture  (see Figure~\ref{fig:bayes_as_NN}). {We have the following equivalent formula for the Bayes estimator.}
\begin{proposition} \label{prop:nnbayes}
{We have that}
\begin{equation}\label{eq:nnbayes}
R_\theta(y)=\sum_{i=1}^L W_i t_i,
\end{equation}
where
\begin{equation*}
    W_i = \softmax(z)_i := \frac{e^{z_i}}{\displaystyle \sum_{j=1}^L e^{z_j}},
\end{equation*}
with
\begin{equation}
\label{z_i}
    z_i = f_i(y) := \log{ \bigg( \frac{w_i}{\sqrt{(2\pi)^n |A \Sigma_i A^T + \Sigma_E|}} \bigg) } -\frac{1}{2} \| (A \Sigma_i A^T + \Sigma_E)^{-\frac{1}{2}} (y-A \mu_i) \|_2^2,
\end{equation}
and 
\begin{equation}
\label{t_i}
    t_i = g_i(y) := \mu_i + \Sigma_i A^T (A \Sigma_i A^T+ \Sigma_E)^{-1} (y-A \mu_i).
\end{equation}    
\end{proposition}

{The proof is a straightforward computation.}

Note that $c_i$ in \eqref{c_i} is $e^{f_i(y)}$ and that $g_i \colon \R^m \to \R^n$ is an affine map in $y$, whereas $f_i \colon \R^m \to \R$ is a \textit{generalized quadratic function} \cite{mantini2021cqnn}, namely
\[
f_i(y) = \gamma_i + b_i^T y + y^T A_i y
\]
for some $\gamma_i \in \R$, $b_i \in \R^m$, $A_i \in \R^{m \times m}$. Second-order functions have already been used as neural network layers  \cite{giles1987learning}.

\begin{figure}
\centering 
    \includegraphics[scale=0.7]{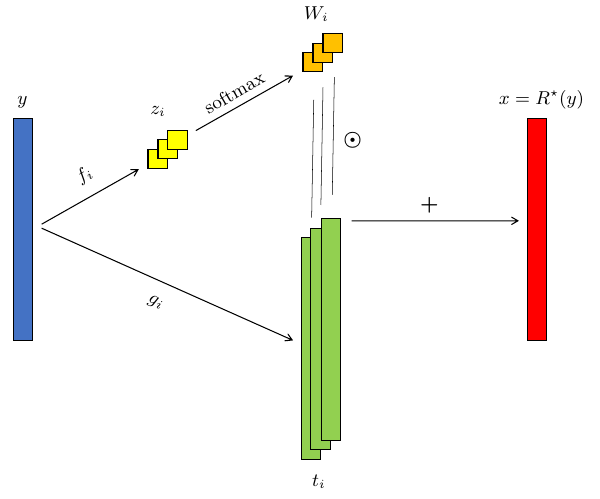}
    \caption{Architecture of the neural network representing the Bayes estimator for denoising with $L=3$.}
\label{fig:bayes_as_NN}
\end{figure}

Adopting the vocabulary of machine learning practitioners, $R_\theta$ is a two-layer feed-forward neural network, with a single hidden layer that involves non-standard operations on the input variables. Such a layer has some connections with an attention mechanism with $L$ channels, as discussed in the next paragraph.

\subsubsection*{Connections with the attention mechanism.}
{The transformer architecture, which uses the attention mechanism \cite{vaswani2017attention}, has greatly improved machine learning, especially in natural language processing (NLP). This mechanism helps models focus on the most important parts of the input, boosting performance in various tasks, including NLP and computer vision. Instead of treating all input elements the same, the attention mechanism gives different weights to different parts, allowing models to allocate resources better and enhance both interpretability and performance.}



More precisely, we consider the self-attention mechanism, which consists of three main components: queries, keys, and values. The queries represent the elements that require attention, while the keys are compared to the queries to determine their relevance. The values correspond to the associated information or representations of the input elements. Attention scores are computed by measuring compatibility or similarity between the queries and the keys, often employing techniques like dot product, additive, or multiplicative attention. Softmax normalization is then used to obtain attention weights. Mathematically, an input $\eta \in \R^{L \times m}$ is linearly transformed into $Q = \eta U_Q$, $K = \eta U_K$ and $V = \eta U_V$, where $U_Q, U_K, U_V \in \R^{m \times m}$ are linear maps. The elements $Q,K,V \in \R^{L \times m}$ represent the $L$ queries, keys, and values of dimension $m$, respectively. The attention mechanism is then
\begin{equation}
\label{att_mech}
\operatorname{Att}(\eta) = \operatorname{Att}(Q,K,V) := \softmax \bigg( \frac{Q K^T}{\sqrt{L}} \bigg) V \in \R^{L \times m},
\end{equation}
where the $\softmax$ acts column-wise, namely
\begin{equation*}
    \softmax(M)_{i,j} = \frac{e^{M_{i,j}}}{\sum_{l=1}^{N_c} e^{M_{i,l}}}, \quad M \in \R^{N_r \times N_c}. 
\end{equation*} 

{In our case, we have the following result, as an immediate consequence of Proposition \ref{prop:nnbayes}.
\begin{proposition}
The Bayes estimator for a Gaussian mixture can be written as follows:
\begin{equation}
\label{att_mech_var}
    R^{\star}(y)  = \overline{\operatorname{Att}}(\eta) = \overline{\operatorname{Att}}(Q,K,V) := \softmax \bigg( (Q \odot K) \mathbf{1}_m^T \bigg)^T V \in \R^{n},
\end{equation}
where $\odot$ represents the Hadamard product (or element-wise product), the matrices $Q,K,V \in \R^{L\times m}$ have rows defined, for $i=1,...,L$, by
\begin{itemize}
    \item $Q_i = \frac{l_i^{\frac{1}{2}}}{\sqrt{m}}  \mathbf{1}_m + \frac{1}{\sqrt{2}} M_i \eta_i \in \R^m$,
    \item  $K_i = \frac{l_i^{\frac{1}{2}}}{\sqrt{m}}  \mathbf{1}_m - \frac{1}{\sqrt{2}} M_i \eta_i \in \R^m$,
    \item $V_i = \mu_i + \Sigma_i A^T M_i^2 \eta_i \in \R^n$,
\end{itemize}
where
\begin{itemize}
    \item $\eta_i = y - A \mu_i \in \R^m$,
    \item $l_i = \log{ \bigg( \frac{w_i}{\sqrt{(2\pi)^n |A \Sigma_i A^T + \Sigma_E|}} \bigg) }$,
    \item $M_i = (A \Sigma_i A^T + \Sigma_E)^{-\frac{1}{2}}$,
    \item $\mathbf{1}_m=(1,\dots,1)\in\R^m$.
\end{itemize}
\end{proposition}
}



Therefore, we have found the queries, the keys and the values to view our network as an alternative version of the attention mechanism. The difference from the classical attention mechanism is the fact that the three transformations for finding $Q,K,V$ are affine, and not linear, and that we use the element-wise product of matrices, instead of the matrix product, which leads to a different size of the network output.

\subsection{(Degenerate) Gaussian mixture as sparsity prior}

The Gaussian mixture prior can be viewed as a sparsity prior by taking $X_i$ in Definition~\ref{Gauss_mix_var} as a degenerate Gaussian. This means that $X_i$ is a Gaussian variable whose support is contained in an $s$-dimensional subspace of $\R^n$, with $s \ll n$ denoting the sparsity level. In other words, the covariance matrix $\Sigma_i$ is degenerate, with $\dim(\ker \Sigma_i)^\perp\le s$. Note that this setting, in which the covariances are not full rank, is compatible with the model considered so far, and with the corresponding estimator written as a neural network, because $\Sigma_E$ is invertible.

Let us briefly discuss why this setting corresponds to a sparsity prior. Take $\mu_i=\textbf{0}$. Write $\Sigma_i$ with respect to its eigenvectors $\{\varphi_k^i\}_k$ and eigenvalues $\{\sigma_k^i\}_k$:
\[
\Sigma_i = \sum_{k=1}^s \sigma_k^i \,\varphi_k^i \otimes \varphi_k^i.
\]
This allows us to expand $X_i$ as
\[
X_i = \sum_{k=1}^s a^i_k \varphi_k^i,
\]
where $a^i_k\sim \mathcal{N}(0,(\sigma_k^i)^2)$. Therefore, with probability $w_i$, we have $X=X_i$, and $X_i$ is a random linear combination of $\varphi^i_1,\dots,\varphi^i_s$, and so is an $s$-sparse vector.

In the case when the number of elements in the mixture, $L$, is equal to the number of all possible subsets of cardinality $s$ of $\{1,\dots,n\}$, $\binom{n}{s}$, and the eigenvectors $\varphi^i_k$ are all chosen from a fixed orthonormal basis $\mathcal{B}$ of $\R^n$, the mixture generates all vectors that are $s$-sparse with respect to $\mathcal{B}$. However, $\binom{n}{s}$ grows very fast in $s$ and $n$, and so this becomes unfeasible even with relatively small values of $s$ and $n$. Therefore, we are led to take $L\ll \binom{n}{s}$, which corresponds to selecting a priori a subset of all possible $s$-dimensional subspaces of $\R^n$, a setting that is commonly referred to as group sparsity \cite{yuan2006model}. On the other hand, we have additional flexibility in the choice of the eigenvectors $\varphi^i_k$, which need not be chosen from a fixed basis.

While Gaussian distributions work well to represent smoothness priors, they are not well-adapted to model sparsity. Here, we propose to use (degenerate) Gaussian mixture models to represent (group) sparsity prior. {In the context of imaging inverse problems, the use of Gaussian Mixture Models to describe structured sparsity has been leveraged in \cite{yu2011solving}, although focusing on maximum-a-posteriori rather than Bayes estimation.} An alternative approach using hierarchical Bayesian models is considered in \cite{2019-calvetti-etal,Calvetti2023}. We also note that our unsupervised approach (introduced below) shares a similar clustering step with the dictionary learning strategy presented in \cite{bocchinfuso2023bayesian} within a hierarchical model framework.

\section{Proposed algorithms: Supervised and Unsupervised approaches}
\label{sec:our_alg}

In this section, we propose two different {techniques} to learn the neural network representation \eqref{eq:nnbayes} of the Bayes estimator \eqref{Bayes_est_gauss_mix} to solve the statistical linear inverse problem of retrieving $X$ from $Y$ in \eqref{statIP}. We assume that the random variables $X$ and $E$ satisfy Assumption~\ref{ass:0}, and write $X$  as a mixture of $L$ random variables \[X = \sum_{i=1}^L X_i \mathbbm{1}_{ \{ i \} }(I),\] where the weights $w_{X_i} = \mathbb{P}(I=i)$, the means $\mu_{X_i} = \E[X_i]$ and the covariances $\Sigma_{X_i} = \operatorname{Cov}[X_i]$ are in general unknown.

We propose two possible regularizers of the form $R_\theta$ (cfr.\ \eqref{eq:theta} and \eqref{eq:hypothesis}), for two ideal choices of parameters $\theta$. The first one, which we will call \textit{supervised}, is
\begin{equation}
\theta^* \in \operatorname{argmin}\big\{L(R_\theta):  \theta \in \Theta, \| \theta \|_{\infty} \leq \varrho \big\},
    \label{eq:optimal_target}
\end{equation}
for some $\varrho >0$, being $L(R_\theta)$ defined as in \eqref{expected_risk} and $\| \theta \|_{\infty}$ the $\ell^{\infty}$ norm of the vectorized $\theta$. Notice that, since the minimization problem \eqref{eq:optimal_target} is restricted to a closed subset of a compact ball and $L(R_{\theta})$ is continuous in $\theta$, the minimum $\theta^*$ exists.

The second parameter choice, which corresponds to an \textit{unsupervised} approach, is instead simply 
\begin{equation}
    \theta_X = (\{w_{X_i}\}_{i=1}^L,\{\mu_{X_i}\}_{i=1}^L,\{\Sigma_{X_i}\}_{i=1}^L),
    \label{eq:thetaX}
\end{equation} 
where the involved quantities are defined above.

We immediately remark that both $\theta^*$ and $\theta_X$ depend on the probability distribution of $X$, which is in general unknown. In Section~\ref{sec:sup_alg}, we show how to approximate $\theta^*$ by taking advantage of a training set of the form $\{ (x_j, y_j) \}_{j=1}^N$ sampled from $(X,Y)$. This qualifies the strategy as a supervised learning algorithm. Instead, in Section~\ref{sec:unsup_alg}, we show how to approximate $\theta_X$ by means of a training set of the form $\{ x_j \}_{j=1}^N$, i.e., in an unsupervised way. 

We also observe that, if we further assume that $E$ is a Gaussian random variable and that $X$ is a mixture of Gaussian random variables,
then the outcomes of the two strategies coincide, provided that  $\|\theta_X\|_\infty\le\varrho$. Indeed, by Theorem~\ref{thm:bayes_est}, we know that $R_{\theta_X}$ is the Bayes estimator $R^*$, and so
\[
R^*=R_{\theta_X}=R_{\theta^*}.
\] 
Nevertheless, the trained networks $R_{\theta^*}$ and $R_{\theta_X}$ could also be good estimators even in contexts that slightly deviate from the setting of Gaussian random variables. In particular, we are interested in testing them in cases in which $X$ is group-sparse, but we are unsure if it is distributed as a Gaussian mixture model. If this is the case, the two parameters are in general distinct - and also different from $R^*$, thus they are not theoretically guaranteed to be good estimators - and the two approximation strategies might achieve different levels of performance.

Notice also that, despite we suppose to deal with the sparsity prior provided by the degenerate Gaussian mixture model, the proposed strategies hold for a general mixture.

\subsection{Supervised approach}
\label{sec:sup_alg}
Let $\{ (x_j, y_j) \}_{j=1}^N$ be a training set, where $x_j\sim X$ are i.i.d., $\epsilon_j\sim E$ are i.i.d.\ and $y_j=Ax_j+\varepsilon_j$. The parameter $\theta = (w_i,\mu_i,\Sigma_i)_{i=1}^L$ of the neural network defined in Section~\ref{sec:bayes_as_NN} can be learned by minimizing the empirical risk
\begin{equation}
    \label{eq:emp}
    \mathcal{L}(\theta) = \frac{1}{N} \sum_{j=1}^N \|x_j - R_{\theta}(y_j) \|_2^2,
\end{equation}
which is the squared MSE between the original signals, $x_j$, and the reconstructions provided by the neural network, $R_{\theta}(y_j)$. 
In order to enforce sparsity, it is possible to add a second term to the empirical risk, namely to consider 
\begin{equation*}
    \mathcal{L}_s(\theta) = \mathcal{L}(\theta) + \lambda \mathcal{J}(\theta),
\end{equation*}
where $\lambda$ is a regularization parameter,  $\mathcal{J}(\theta) = \sum_{i=1}^L \| \Sigma_i \|_*$, and $\| \cdot \|_*$ is the nuclear norm. 
Such a regularization term can be equivalently defined as the $\ell^1$ norm of the singular values of $\Sigma_i$, hence it promotes low-rank covariances without requiring the knowledge of their eigenvectors. 

The evaluation of the nuclear norm is computationally expensive and should be done $L$ times for each step of the minimization process. Another option is taking $\mathcal{J}(\theta) = \sum_{i=1}^L \| \Sigma_i \|^2_{\text{F}}$, where $\| \cdot \|_{\text{F}}$ represents the Frobenius norm. This choice is computationally convenient, but does not promote any kind of sparsity on the covariances. However, in the numerical simulations, we do not observe significant differences when using the nuclear norm, the Frobenius norm or no regularization term {(see in particular Table \ref{tab_3})}. Therefore, in the numerical results of Section~\ref{sec:numerics} we did not make use of any regularization term.


\subsection{Unsupervised approach}
\label{sec:unsup_alg}

Suppose to be in an unsupervised setting, i.e.\ to have a training set $\{ x_j\}_{j=1}^N$, where $x_j$ are sampled i.i.d.\ from the mixture $X = \sum_{i=1}^L X_i \mathbbm{1}_{ \{i \} }(I)$. We now wish to find an approximation for the means $\mu_{X_i}$, the covariances $\Sigma_{X_i}$ and the weights $w_{X_i}$ of the mixture.  We propose the following two-step procedure.

\begin{enumerate}
    \item \textbf{Subspace clustering}. The elements of the training set $\{ x_j\}_{j=1}^N$ sampled from different degenerate distributions $X_i$  belong to different subspaces of $\R^n$. We propose to cluster these points by using subspace clustering and, in particular, the technique provided in \cite{lu2012robust}. As a result, the training set is partitioned into $\widehat L$ subsets. Ideally, we should have $\widehat L=L$, and each subset should correspond to one Gaussian of the mixture.
    \item \textbf{The parameters}. We compute the empirical means $\{\widehat\mu_i\}_{i=1}^{\widehat L}$ and the empirical covariances $\{\widehat\Sigma_i\}_{i=1}^{\widehat L}$ of each cluster, and we estimate the weights $\{\widehat w_i\}_{i=1}^{\widehat L}$ of the mixture by using the number of elements in the clusters.
\end{enumerate}
Once this is done, we can use the neural network defined in Section~\ref{sec:bayes_as_NN} with these parameters. More precisely, this is the neural network $R_{\widehat \theta_X}$, where
\[
\widehat \theta_X = \big( \{\widehat w_i\}_{i=1}^{\widehat L}, \{\widehat\mu_i\}_{i=1}^{\widehat L}, \{\widehat\Sigma_i\}_{i=1}^{\widehat L}\big). 
\]

\section{Baseline algorithms}
\label{sec:comp}
In this section, we describe some of the most popular regularization techniques used to solve linear inverse problems with a sparsity prior. These include Least Absolute Shrinkage and Selection Operator (LASSO), Group LASSO, Iterative Hard Thresholding (IHT), and Dictionary Learning. We report here the formulation and general idea of such methods and postpone the discussion on the implementation aspects to Appendix \ref{app:algorithms}. In the numerical experiments of Section \ref{sec:numerics}, we will compare our two proposed approaches against the methods presented in this section.

\subsection{LASSO}
\label{sec:LASSO}

Least Absolute Shrinkage and Selection Operator (LASSO) \cite{tibshirani1996regression}, is one of the most acknowledged techniques for sparsity promotion, also in the context of inverse problems. It involves the minimization of a functional composed by the sum of a data fidelity term and a regularization term that promotes sparsity with respect to a given basis. In particular, the regularization term is the $\ell^1$ norm of the 
components of the unknown in the sparsifying basis.

More precisely, we suppose that the unknown $x$ in \eqref{IP} is sparse with respect to an orthonormal basis $\mathcal{B}$, and we denote by $M \in \R^{n \times n}$ the (orthogonal) matrix representing the change of basis from $\mathcal{B}$ to the canonical one, also known as the synthesis operator. Then, the functional to be minimized is
\begin{equation}
\label{LASSO_func}
\mathcal{F}(\beta) = \frac{1}{2} \| y - A M \beta \|_2^2 + \lambda \| \beta \|_1,
\end{equation}
where $\lambda > 0$ is a regularization parameter. The minimization of $\mathcal{F}(\beta)$ over $\R^n$, also known as synthesis formulation of LASSO, is discussed in Appendix~\ref{app:algorithms}.
The main drawbacks of this method are the choice of the regularization parameter $\lambda$, and the required prior knowledge of the synthesis operator $M$ (i.e., of the basis $\mathcal{B}$ with respect to which the unknown is sparse). 
In Section \ref{sec:dict_learn}, we show how to combine dictionary learning techniques to fill this gap. Other alternatives, explored in the numerical experiments in Section \ref{sec:numerics}, are to leverage prior knowledge on $M$, or to infer it via the singular value decomposition (SVD).

\subsection{Group LASSO}
\label{sec:Glasso}
An extension of LASSO is Group LASSO \cite{friedman2010note}, which encodes the idea that the features of the possible unknowns can be organized into groups. This is achieved by 
introducing $L$ different coordinate systems, each represented by an orthogonal synthesis matrix $M_i \in \R^{n \times n}$, and by minimizing the following functional:
\begin{equation}
\label{GLASSO_func_gen}
\mathcal{F}(\beta) = \frac{1}{2} \bigg\| y - A \sum_{i=1}^L M_i \beta_i \bigg\|_2^2 + \lambda \sum_{i=1}^L  \| \beta_i \|_{K_i},
\end{equation}
where $\beta_i$ is the coefficient vector corresponding to the $i$-th group, $\beta \in \R^{nL}$ is the concatenation of the vectors $\beta_1,\ldots,\beta_L$, and $\lambda > 0$ is a regularization parameter. 
The second term of \eqref{GLASSO_func_gen} is a regularization functional encouraging entire groups of features to be either included or excluded from the model. The weighted norm $\| \beta_i \|_2$ with $\| \beta_i \|_{K_i}:= (\beta_i^T K_i \beta_i)^{\frac{1}{2}}$ is employed to promote prior information regarding the sparse representation of each group.
The minimization of \eqref{GLASSO_func_gen} can be performed via the iterative algorithm proposed in \cite[Proposition $1$]{yuan2006model}, as discussed in Appendix \ref{app:algorithms}.
As for LASSO, also for Group LASSO one has to choose the regularization parameter $\lambda$, and to know a priori the group bases $\mathcal{B}_i$ for $i=1,...,L$.

\subsection{IHT}
\label{sec:IHT}

The Iterative Hard Thresholding algorithm, as presented in \cite{blumensath2013compressed}, can be employed as a reconstruction method for inverse problems in which the unknown belongs to the union of $L$ subspaces. This can be seen as a model of sparsity. Once an orthogonal synthesis matrix $M \in \R^{n \times n}$ is introduced, the coefficients $\beta$ such that $x = M\beta$ are supposed to satisfy $\beta \in \mathcal{S} = \cup_{i=1}^L \mathcal{S}_i$, where each $\mathcal{S}_i$ is a coordinate subspace in $\R^n$ of dimension $s$ or smaller, i.e., the span of a subset of cardinality at most $s$ of the canonical basis. IHT then involves the minimization of
\begin{equation}
\label{IHT_func}
\mathcal{F}(\beta) = \frac{1}{2} \| y - AM\beta \|_2^2 + \chi_{\mathcal{S}}(\beta),
\end{equation}
where 
\begin{equation*}
    \chi_{\mathcal{S}}(\beta) = \begin{cases}
        0 & \quad \beta \in \mathcal{S}, \\
        + \infty & \quad \text{otherwise,}
    \end{cases}
\end{equation*}
forces the unknown to belong to $\mathcal{S}$.

As for LASSO and Group LASSO, the subspaces $\mathcal{S}_i$, and especially the basis given by $M$, should be known a priori or should be inferred. In this setting, the sparsity level $s$ can play the role of a regularization parameter (such as $\lambda$ of LASSO and Group LASSO) and is tuned by suitable heuristic methods.

\subsection{Dictionary Learning}
\label{sec:dict_learn}
The baseline algorithms proposed in the previous sections need the a priori knowledge of the basis with respect to which the unknown is sparse. A possible approach to infer it is \textit{sparse dictionary learning} (also known as sparse coding), where the dictionary is learned from the data \cite{mairal2009online}. Note that this dictionary does not necessarily have to be a basis. More precisely, given a training set $\{ x_j\}_{j=1}^N$ with $x_j \in \R^n$, a sparsifying dictionary $D \in \R^{n \times d}$, where $d$ is the number of elements of the dictionary (the columns of the matrix $D$), can be found as
\begin{equation}
\label{dict_learn_func}
    \min_{D \in \R^{n \times d},\, \beta \in \R^{d \times N}} \frac{1}{N} \sum_{j=1}^{N} \bigg( \frac{1}{2} \| x_j - D \beta_j \|_2^2 + \lambda \| \beta_j \|_1 \bigg),
\end{equation}
where $\lambda > 0$ is a regularization parameter and $\beta_j\in\R^d$ is the sparse representation of $x_j$ with respect to the dictionary $D$. If the dictionary is fixed, the functional in \eqref{dict_learn_func} resembles the LASSO functional \eqref{LASSO_func}, except for the fact that the dictionary may not be a basis and that we are summing over the elements of the training set. However, in this case, the goal is mainly to learn $D$. In order to solve the minimization problem in \eqref{dict_learn_func}, we rely on the \textit{online method} used in \cite{mairal2009online}, which iteratively minimizes over one variable keeping the other ones fixed. We update the sparse representation $\beta$ using the least angle regression algorithm (LARS) and the dictionary $D$ by block coordinate descent. 
As reported in \cite{mairal2009online}, the computational cost of each iteration of the online method is dominated by the update of $\beta$, and the complexity of LARS applied on $D \in \R^{n \times d}$ is of order $d^3 + d^2 n$ (see \cite{efron2004least}).

We assume that $D \in \R^{n \times d}$ is a full-rank matrix. In general, the number of atoms $d$ is allowed to be larger than the original dimension of the signal $n$. Such a redundancy may arise with frames or with the union of multiple bases. Nevertheless, here we consider the non-redundant setting, with $d \leq n$. This is the case for all the numerical experiments reported in this work. 

Once the dictionary is learned, the solution to the linear inverse problem \eqref{statIP} can be computed by combining the baseline algorithms of the previous sections. 
Under the assumptions previously introduced, the matrix $D$ is injective, and the pseudo-inverse $D^+ = (D^T D)^{-1} D^T$ is its left inverse.  We thus tackle the minimization of the following functional:
\begin{equation}
    \label{DL_func}
    \mathcal{F}(x) = \frac{1}{2} \|y-Ax\|^2 + \lambda G(x), \qquad G(x) = \chi_{\operatorname{Im}(D)}(x) + \|D^{+}x\|_1,
\end{equation}
where $\chi_{\operatorname{Im}(D)}(x)$ is the characteristic function of the range $\operatorname{Im}(D)$.

Since the degenerate Gaussian mixture prior represents a group sparsity prior, it is useful to compare our algorithms with a ``group'' version of sparse dictionary learning. This has been explored with the name ``Block-sparse'' or ``Group-sparse'' dictionary learning \cite{zelnik2012dictionary,li2012group}.

Those techniques, however, do not share the same perspective on group sparsity as the one inspiring our proposed algorithms. The common feature behind all of them is the assumption that, when the signals are represented in a suitable basis or dictionary, there exists groups, or blocks, of coordinates that are simultaneously activated. In our degenerate Gaussian mixture approach, though, the signals can be clustered according to which group of coordinates they activate: each signal is thus associated with a single group. In block-sparse dictionary learning (and, similarly, in the Group LASSO approach previously exposed), the active components of each signal can also belong to (a small number of) different groups.

For this reason, we propose an alternative dictionary learning technique that is closer to our setup and can be employed for more direct comparisons.
We simply call it ``Group Dictionary Learning'':  after doing a subspace clustering of the training set as explained in Section~\ref{sec:unsup_alg}, it is possible to learn a dictionary $D_i$ for each group obtained with the clustering. 
Assuming that each $D_i \in \R^{n \times d_i}$ is injective for $i=1,...,L$, and adopting the same approach as in \eqref{DL_func}, we tackle the minimization of the following functional:
\begin{equation}
\label{eq:GDL}
    \begin{gathered}
    \mathcal{F}(x) = \frac{1}{2}\| y-Ax \|^2 + \lambda \hat{G}(x), \\
    \hat{G}(x) = \min_{i = 1,\ldots,L}G_i(x), \qquad  G_i(x) = \chi_{\operatorname{Im}(D_i)}(x) + \|D_i^{+}x\|_1.
\end{gathered}
\end{equation}

\begin{rem}
Differently from dictionary learning, the approaches proposed in Section~\ref{sec:our_alg} do not focus on the reconstruction of a dictionary $D$. However, such a dictionary might be obtained as a by-product, e.g.\ by computing the singular vectors of each covariance $\Sigma_i$ and by collecting them in a single matrix, but the theoretical properties of such an object are out of the scope of the present work. Nevertheless, we wish to underline that, as in the problem of dictionary learning, our algorithms do not require the knowledge of the dictionary $D$ as an input. It is worth noting that the parameters $\theta = (w_i, \mu_i, \Sigma_i)_{i=1}^L$ of our network are $O(n^2 \times L)$ while the dictionary has only $n \times d$ parameters, where $d$ is usually chosen as $O(n)$. Therefore, one of the reasons our methods seem more effective may be that they employ more parameters, a well-known (but possibly not fully understood) phenomenon related to overparametrization in deep learning.
\end{rem}

\section{Numerical results}
\label{sec:numerics}
In this section, we compare the {methods} proposed in Section~\ref{sec:our_alg} with the baseline algorithms discussed in Section~\ref{sec:comp} for $1$D denoising and deblurring problems with three datasets. Our experiments primarily aim to compare our {methods} with classical sparsity-promoting techniques on simplified signal classes, rather than striving for state-of-the-art results. The exploration of more complex inverse problems and the incorporation of real-world data fall outside the scope of this paper and are designated for future research.

\subsection{Datasets}
\label{subsec:datasets}
We consider three datasets with increasing complexity.
\begin{enumerate}
    \item \label{GMM_10} \textit{Gaussian mixture model}: 
    This dataset contains samples from a degenerate Gaussian mixture variable \eqref{Gauss_mix_var} where $I$ is a uniform variable (all the Gaussians of the mixture have the weight $w_i=\frac1L$).
    Each $X_i$ is a Gaussian random variable in $\R^n$, having mean $\mu_i = \textbf{0}$ and whose covariance matrix $\Sigma_i$ has zero entries everywhere, except from $s$ elements on the diagonal which are set to $1$. This indicates that we are considering a mixture of degenerate variables, each of which has support on an $s$-dimensional coordinate hyperplane of $\R^n$ and whose components are independent standard Gaussian random variables.
    In our experiments we consider $n=1000$, $s=\operatorname{rank}\Sigma_i = 20$ and $L=10$. Therefore, this dataset consists of very sparse random variables in a rather large ambient space. The sparsity is nevertheless very structured, since only $L=10$ combinations of $s$ active coefficients are considered.

    \item \label{S1D} \textit{Sinusoidal functions with one discontinuity}:
    This dataset contains functions with support in $[0,4 \pi]$ of the type 
    \begin{equation*}
        x(\tau) = \begin{cases}
            A \sin(\omega \tau) + B & \quad 
            0 \leq \tau \leq \tau_i \\
            A \sin(\omega \tau) + B + C & \quad \tau_i < \tau \leq 4 \pi
        \end{cases}
    \end{equation*}
    where the amplitude $A \sim \mathcal{U}(0.05, 0.1)$, the angular frequency $\omega \sim \mathcal{U}(1,2)$ and the vertical translations $B \sim \mathcal{U}(\frac{1}{2}, 3)$, $C \sim \mathcal{N}(0,0.2^2)$. The discontinuity points $\tau_i$ with $i=1,...,L$ and $L=10$ correspond to different subspaces and are equispaced points in $[0,4 \pi]$. Each signal $x$ is discretized on $1000$ equispaced points in $[0,4 \pi]$. These functions can be approximately seen as samples from a degenerate Gaussian mixture in a wavelet basis. While the coarse scale coefficients are expected to be nonzero for all the signals, the wavelet coefficients at the finer scales will be relevant only in the locations of the discontinuity, and negligible in the smooth regions. Therefore, each discontinuity determines a coordinate subspace with respect to the wavelet basis. Considering that the amplitude of the jump $C$ follows a Gaussian distribution, we posit that the fine-scale coefficients can be effectively modeled as degenerate Gaussian random vectors. We can also infer an estimate of the sparsity level in terms of the size of the support of the mother wavelet, and of the considered resolution scales.
    
    \item \label{FS2D} \textit{Truncated Fourier series with two discontinuities}:
    This dataset contains functions with support in $[0,4 \pi]$ of  the form 
    \begin{equation*}
        x(\tau) = \begin{cases}
            \sum_{d=1}^4 a_d \cos{(2 \pi d \tau)} + b_d \sin({2\pi d \tau)} & \quad 0 \leq \tau \leq \tau_{i}(1) \\
            \sum_{d=1}^4 a_d \cos{(2 \pi d \tau)} + b_d \sin({2\pi d \tau)} + C_1 & \quad \tau_{i}(1) < \tau \leq \tau_{i}(2) \\
            \sum_{d=1}^4 a_d \cos{(2 \pi d \tau)} + b_d \sin({2\pi d \tau)} + C_2 & \quad \tau_{i}(2) < \tau \leq 4 \pi 
        \end{cases}
    \end{equation*}
    where the Fourier coefficients $a_d,b_d \sim \mathcal{N}(0.1,0.1^2)$ for $d=1,...,4$ and the vertical translations $C_1,C_2 \sim \mathcal{N}(0,0.2^2)$. As for dataset~\ref{S1D}, the discontinuity points are $10$ equispaced points in $[0,4 \pi]$ and $\tau$ is discretized with $1000$ equispaced points in $[0,4 \pi]$. However, each function has either one or two discontinuities. More precisely, if $\tau_i(1) = \tau_i(2)$ the function has one discontinuity, otherwise it has two discontinuities. Therefore the subspaces are the ones representing functions with $2$ discontinuities (all the possible combinations of $10$ elements in groups of $2$) and the ones representing functions with $1$ discontinuity ($10$ subspaces), so the total number of subspaces is $L = \binom{10}{2}+10 = 55$. For the same reasons discussed above for dataset~\ref{S1D}, these functions can be seen as approximate samples from a degenerate Gaussian mixture distribution in the wavelet domain. The level of sparsity can be estimated as in the case of the previous dataset, updating it to the presence of two distinct singularities.
\end{enumerate}

\subsection{Methods}
\label{subsec:methods}
In the following experiments, we test our supervised and unsupervised {approaches} described in Sections~\ref{sec:sup_alg} and \ref{sec:unsup_alg} and compare them with the baseline algorithms described in Section~\ref{sec:comp}. More precisely, we consider the following methods.
\begin{enumerate}[label=(\Alph*)]
    \item \label{item:sup} \textit{Supervised}($\S$\ref{sec:sup_alg}). {We consider both the case in which the weights of the network are randomly initialized, denoted as (A), and the one in which they are initialized by the values obtained as an outcome of the unsupervised procedure, denoted as (A+B).}
    
    \item \label{item:unsup} \textit{Unsupervised }($\S$\ref{sec:unsup_alg}). 
    
    \item \label{item:DL} \textit{Dictionary learning} ($\S$\ref{sec:dict_learn}). We learn a sparsifying dictionary $D$ for the training set, using the Lasso LARS algorithm to solve \eqref{dict_learn_func}; then, we employ the learned $D$ to denoise the test signal by minimizing \eqref{DL_func}. We choose the number of elements of the dictionary as $d = \frac{n}{2}$, where $n$ is the signal length, in order to balance expressivity and numerical efficiency.
    
    \item \label{item:GDL} \textit{Group dictionary learning} ($\S$\ref{sec:dict_learn}). 
    After dividing the training set into $L$ clusters and learning a sparsifying dictionary $D_i$ for each group, we minimize \eqref{eq:GDL}. In this case, we choose $d = \frac{n}{2L}$.
    
    \item \label{item:IHT_SVD} \textit{IHT with SVD basis} ($\S$\ref{sec:IHT}). We infer the basis with respect to which the unknown is sparse by computing the SVD of the empirical covariance matrix of the training set and by considering the orthonormal basis composed by the eigenvectors. Then, we choose $\mathcal{S}$ as the union of all the $s$-dimensional coordinate subspaces in $\R^n$ w.r.t.\ the inferred sparsity basis. Here we choose the degree of sparsity $s$ by minimizing the relative MSE (computed w.r.t.\ the norm of the original signals) over the training set.   
    
    \item \label{item:GIHT_SVD} \textit{IHT with SVD bases of groups} ($\S$\ref{sec:IHT}). We divide the training set into $L$ clusters as explained in Section~\ref{sec:unsup_alg}. Then, for each group provided by the clustering, we infer the basis with respect to which that group is sparse by computing the SVD of its empirical covariance matrix and by considering the eigenvectors corresponding to the $s$ largest eigenvalues, where the degree of sparsity $s$ is chosen as for \ref{item:IHT_SVD}. Finally, we choose $\mathcal{S}$ as the union of the $s$-dimensional subspaces spanned by the bases inferred. 
    
    \item \label{item:IHT_known} \textit{IHT with known basis} ($\S$\ref{sec:IHT}). We suppose to know the basis with respect to which the unknown is sparse. In particular, for Dataset~\ref{GMM_10} we consider the canonical basis, while Datasets~\ref{S1D} and \ref{FS2D} require a more complicated treatment. Indeed, as discussed in Section~\ref{subsec:datasets}, their representation in a suitable wavelet basis is sparse at fine scales. For this reason, we preprocess the datasets by applying the wavelet transform, keep the low-scale coefficients fixed, and apply IHT only to the fine scales. We choose the wavelet basis generated by the Daubechies wavelet with $6$ vanishing moments. 
    
    \item \label{item:LASSO_SVD} \textit{LASSO with SVD basis} ($\S$\ref{sec:LASSO}). We infer the basis with respect to which the unknown is sparse as in \ref{item:IHT_SVD}. Then, we minimize \eqref{LASSO_func}.
    
    \item \label{item:GLASSO_SVD} \textit{Group LASSO with SVD bases} ($\S$\ref{sec:Glasso}). We divide the training set into $L$ clusters as explained in Section~\ref{sec:unsup_alg}. Then, for each group provided by the clustering, we infer the basis with respect to which that group is sparse as explained in \ref{item:GIHT_SVD}. Here, however, we consider the complete bases provided by the SVDs of the empirical covariance matrices of the groups. We choose the power $-\frac{1}{2}$ of the empirical covariance matrices of each group as penalty matrices $K_i$. 
    
    \item \label{item:LASSO_known} \textit{LASSO with known basis} ($\S$\ref{sec:LASSO}). 
    We incorporate in LASSO the knowledge of the basis with respect to which the unknown is sparse, as already explained for \ref{item:IHT_known}.
    This implies that, for Dataset~\ref{GMM_10}, we solve LASSO with the canonical basis, whereas for Datasets~\ref{S1D} and \ref{FS2D} we first compute the wavelet transform and then solve LASSO only on the high-scale coefficients.

\end{enumerate} 

Whenever needed (methods \ref{item:DL}, \ref{item:GDL}, \ref{item:LASSO_SVD}, \ref{item:GLASSO_SVD}, and \ref{item:LASSO_known}), we always choose the optimal regularization parameter $\lambda$, namely, the one minimizing the relative MSE over the training set. Further, whenever needed (methods \ref{item:sup}, \ref{item:unsup}, \ref{item:GDL}, \ref{item:IHT_SVD}, and \ref{item:GLASSO_SVD}), we always suppose to know the number of Gaussians $L$ in the mixture.

Methods \ref{item:unsup}, \ref{item:GDL}, \ref{item:IHT_SVD}, and \ref{item:GLASSO_SVD} rely on clustering. The subspace clustering procedure described in Section \ref{sec:unsup_alg} is applied to the signals for Dataset~\ref{GMM_10} and to their derivatives computed as finite difference approximations for Datasets~\ref{S1D} and \ref{FS2D}. Indeed, for the latest datasets, the derivative of the signals highlights the discontinuity points and allows for a more efficient clustering.

It is worth noting that methods \ref{item:IHT_known} and \ref{item:LASSO_known} take advantage of the knowledge of the basis with respect to which the unknown is sparse. For this reason, they do not constitute a fair comparison for all the other algorithms, which do not use this piece of information. However, we are interested to see if they can outperform our methods.

\subsection{Denoising}
\label{sec:denois_comp}
In this section, we focus on the denoising problem with $10\%$ of noise, namely $y = x + \epsilon$, where $\epsilon$ is sampled from $\mathcal{N}(\mathbf{0},\sigma^2 I)$ and $\sigma$ is the $10\%$ of the maximum of the amplitudes of the training set signals, i.e.\ $\max_f (\max f - \min f)$ where $f$ is a training sample and the maximum is taken for each dataset. We compare the performances of the methods reported in Section \ref{subsec:methods}.

In Table~\ref{tab_denoising} we show the mean over $2000$ signals of the test set of the relative MSE between the original signals and the different reconstructions. We observe that our unsupervised {technique} \ref{item:unsup} provides the best reconstructions for Datasets~\ref{GMM_10} and \ref{FS2D}, and for Dataset~\ref{S1D} it is only outperformed by group dictionary learning \ref{item:GDL}. However, dictionary learning is computationally more expensive than our unsupervised method. 
Indeed, although both techniques involve a preliminary clustering step, group dictionary learning also requires computing $L$ sparsifying dictionaries, i.e., solving $L$ optimization problems as \eqref{dict_learn_func}. In our experiments, for each dictionary, the online method presented in Section \ref{sec:dict_learn} performed nearly $300$ iterations, with the cost of a single update depending on the size of both the dictionary and the sample, as discussed in Section \ref{sec:dict_learn}. On the other hand, the unsupervised technique only requires estimating $L$ means and covariances prior to the application of \eqref{eq:nnbayes}, and the difference in performance is minimal
(as it can be seen in Figure~\ref{fig:denoising_main_body}). From Table~\ref{tab_denoising} it seems clear that the best baseline algorithms to which we can compare our approaches are dictionary learning \ref{item:DL}, among the methods that do not use clustering, and group dictionary learning \ref{item:GDL}, among the methods using the groups obtained by the clustering of the training set. For this reason, in Figure~\ref{fig:denoising_main_body} we show a qualitative comparison between the reconstructions provided by our methods, dictionary learning, and group dictionary learning. For completeness, we provide the remaining qualitative comparisons in Figure~\ref{fig:denoising_appendix} of Appendix~\ref{app:numerics}. We notice that for Datasets~\ref{S1D} and \ref{FS2D} the algorithms learn all the discontinuities and try to remove those not present in the signal.

\begin{table}
\centering 
\caption{Relative MSE values for the denoising problem with $10 \%$ noise.}
\label{tab_denoising}
\begin{tabular}{l|ccc}
&Dataset~\ref{GMM_10}&Dataset~\ref{S1D}&Dataset~\ref{FS2D} \\\hline
Supervised \ref{item:sup} & $\num{1.97e0} \%$ & $\num{7.03e-3} \%$ & $\num{2.71e-2} \%$ \\\hline
{Supervised \hyperref[item:sup]{(A+B)} }& $\mathbf{\num{0.98e0}} \%$ & $\num{1.79e-3} \%$ & $\num{6.38e-3} \%$ \\\hline
Unsupervised \ref{item:unsup} & $\mathbf{\num{0.98e0}} \%$ & $\num{1.78e-3} \%$ & $\mathbf{\num{6.32e-3}} \%$ \\\hline
Dictionary learning \ref{item:DL} & $\num{4.18e0} \%$ & $\num{3.43e-3} \%$ & $\num{7.13e-3} \%$  \\\hline
Group dictionary learning \ref{item:GDL} & $\num{0.99e0} \%$ & $\mathbf{\num{1.70e-3}} \%$ & $\num{2.42e-2} \%$  \\\hline
IHT with SVD basis \ref{item:IHT_SVD} & $\num{9.93e0} \%$ & $\num{1.25e-2} \%$ & $\num{3.97e-1} \%$  \\\hline
IHT with SVD bases of groups \ref{item:GIHT_SVD} & $\num{0.99e0} \%$ & $\num{2.04e-3} \%$ & $\num{3.97e-1} \%$  \\\hline
IHT with known basis \ref{item:IHT_known} & $\num{2.78e0} \%$ & $\num{1.22e-2} \%$ & $\num{2.46e-2} \%$  \\\hline
LASSO with SVD basis \ref{item:LASSO_SVD} & $\num{9.24e0} \%$ & $\num{1.55e-2} \%$ & $\num{3.07e-1} \%$  \\\hline
Group LASSO with SVD bases \ref{item:GLASSO_SVD} & $\num{3.01e0} \%$ & $\num{3.71e-3} \%$ & $\num{8.31e-1} \%$  \\\hline
LASSO with known basis \ref{item:LASSO_known} & $\num{4.69e0} \%$ & $\num{1.00e-2} \%$ & $\num{2.15e-2} \%$  \\\hline
\end{tabular}
\end{table}

\begin{figure}
\centering
\includegraphics[scale=0.88]{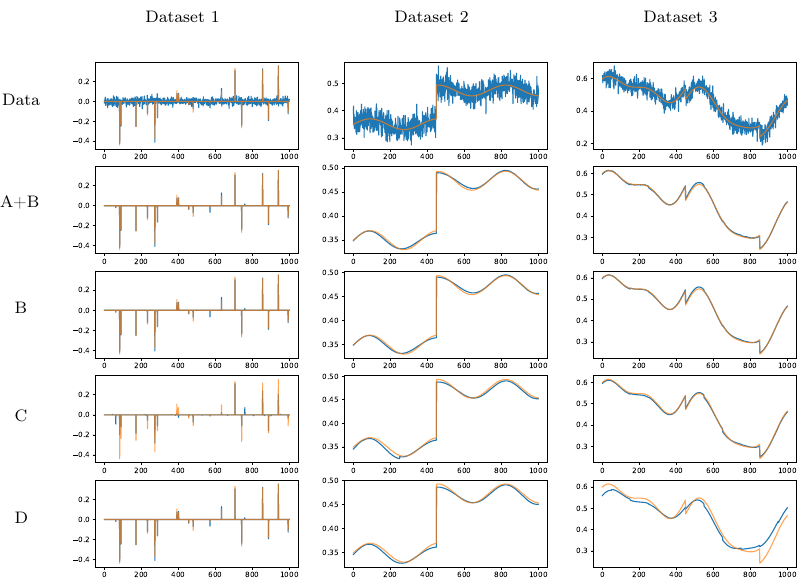}

\caption{Qualitative comparison for the denoising problem. In each column we show a signal of the test set from Datasets~\ref{GMM_10}, \ref{S1D} and \ref{FS2D}, respectively. In each row we report the original data in orange and the noisy data, the reconstructions obtained with the methods {supervised (initialized with the unsupervised one) \hyperref[item:sup]{(A+B)}}, unsupervised \ref{item:unsup}, dictionary learning \ref{item:DL}, and group dictionary learning \ref{item:GDL}, respectively, in blue.} 
\label{fig:denoising_main_body}
   
\end{figure}

\subsubsection*{Does clustering really matter?}

Considering that accurately clustering the training set is challenging (see e.g.\ Dataset~\ref{FS2D}),  we question whether precise clustering is truly necessary for employing the unsupervised algorithm proposed in Section~\ref{sec:unsup_alg}. For this purpose, in Table~\ref{tab_2} we show the mean over $2000$ signals of the test set of the relative MSE between the original signals and the reconstructions provided by the unsupervised approach with exact clustering and with random clustering. 
We observe that correct clustering is important for Dataset 1, but not so important for Datasets 2 and 3. We believe that this is due to the fact that in Datasets 2 and 3 the energy of the signals is shared between the smooth part, which is independent of the clustering, and the discontinuities. In Dataset 1, the absence of the (non-zero) smooth part makes the dependence on the clustering stronger.

\begin{table}
\centering 
\caption{Relative MSE values for the denoising problem with $10 \%$ noise with the unsupervised approach \ref{item:unsup} when using exact vs. random clustering. } 
\label{tab_2}
\begin{tabular}{l|ccc}
&Dataset~\ref{GMM_10}&Dataset~\ref{S1D}&Dataset~\ref{FS2D} \\\hline
Exact & $\num{0.97e0} \%$ & $\num{1.66e-3} \%$ & $\num{3.37e-3} \%$ \\\hline 
Learned & $\num{0.98e0}\%$ & $\num{1.80e-3} \%$ & $\num{6.32e-3} \%$ \\\hline
Random & $ \num{8.25e0} \pm \num{0.04} \%$ & $(3.46\pm 0.25)\times 10^{-3}\%$ & $(5.70 \pm 0.55)\times 10^{-3} \%$ \\\hline
\end{tabular}
\end{table}

{
\subsubsection*{Does regularization really matter?} We investigated the impact of the regularization term $\mathcal{J}$ added to the training loss in the supervised approach, as discussed in Section \ref{sec:sup_alg}. Table \ref{tab_3} shows that, when applied to Dataset 1, the use of the Frobenius or the nuclear norms does not improve the results in terms of MSE. In both cases, the value of the regularization parameter is heuristically chosen. Notice that also the matrices $\Sigma_i$ trained without any penalty term naturally show a good degree of sparsity (quantified by the ratio between their average rank and the size of the signal). }
\begin{table}
\centering 
\caption{{Denoising problem with $10 \%$ noise for Dataset 1 with the supervised approach \ref{item:unsup} without penalty, with a nuclear-norm penalty, and with a Frobenius-norm penalty. We compare the values of the relative MSE and the ratio between the average rank of the learned $\Sigma_i$ and the signal size $n$.\\} 
}
\label{tab_3} {
\begin{tabular}{l|ccc}
&No regularization & Nuclear norm & Frobenius norm \\\hline
Relative MSE & $1.97 \%$ & $2.18 \%$ & $2.11 \%$ \\\hline 
Average rank of $\Sigma_i$ & 0.41 $n$ & 0.22 $n$ & 0.92 $n$ \\\hline
\end{tabular}}
\end{table}

\subsection{Deblurring}
\label{sec:deblur_comp}
In this section, we focus on a deblurring problem with $10\%$ of noise, namely we consider the problem $y = q \ast x + \epsilon$, where $\ast$ represents the discrete convolution, the noise $\epsilon$ is sampled from $\mathcal{N}(\mathbf{0},\sigma^2 I)$ and $\sigma$ is the $10\%$ of the maximum of the amplitudes of the training set signals. We choose the filter $q$ to be Gaussian, i.e.\ the entries of $q$ are a finite discretization of the density of $\mathcal{N}(0,\sigma_b^2)$ in a neighborhood of $0$. For Dataset~\ref{GMM_10} we choose $\sigma_b = 1$, while for Datasets~\ref{S1D} and \ref{FS2D} we set $\sigma_b = 30$ and $\sigma_b = 20$, respectively. The latter values of $\sigma_b$ are larger because the effect of the convolution on piecewise smooth signals (Datasets~\ref{S1D} and \ref{FS2D}) is less visible than on Dirac deltas (Dataset~\ref{GMM_10}). Since for the denoising problem our unsupervised method provides better results than the supervised one, we only show the results for the deblurring problem using the unsupervised technique and we compare it with the most significant baseline algorithms. In particular, we consider the following methods: Unsupervised \ref{item:unsup}, Dictionary learning \ref{item:DL}, Group dictionary learning \ref{item:GDL}, IHT with SVD bases of groups \ref{item:GIHT_SVD} and Group LASSO with SVD bases \ref{item:GLASSO_SVD}. 
We exclude from the comparisons the methods that do not employ clustering, except for dictionary learning, as they performed significantly worse in the denoising experiments.

In Table~\ref{tab_deblurring} we show the mean over $2000$ signals of the test set of the relative MSE between the original signals and the different reconstructions. As for the denoising problem, our unsupervised method outperforms the others.

In Figure~\ref{fig:deblurring_main_body} we show a qualitative comparison between the reconstructions provided by our methods, dictionary learning, and group dictionary learning. For completeness, we provide the remaining qualitative comparisons in  Appendix~\ref{app:numerics}.

\begin{table}
\centering 
\caption{Relative MSE values for the deblurring problem with $10 \%$ noise.}
\label{tab_deblurring}
\begin{tabular}{l|ccc}
&Dataset~\ref{GMM_10}&Dataset~\ref{S1D}&Dataset~\ref{FS2D} \\\hline
Unsupervised \ref{item:unsup} & $\mathbf{\num{3.68e0}} \%$ & $\mathbf{\num{2.65e-3}} \%$ & $\mathbf{\num{1.01e-2}} \%$ \\\hline
Dictionary learning \ref{item:DL} & $\num{14.32} \%$ & $\num{6.61e-3} \%$ & $\num{1.28e-2} \%$  \\\hline
Group dictionary learning \ref{item:GDL} & $\num{13.51} \%$ & $\num{4.62e-3} \%$ & $\num{3.41e-2} \%$  \\\hline
IHT with SVD bases of groups \ref{item:GIHT_SVD} & $\num{3.80} \%$ & $\num{5.54e-3} \%$ & $\num{9.48e-1} \%$  \\\hline
Group LASSO with SVD bases \ref{item:GLASSO_SVD} & $\num{11.48} \%$ & $\num{1.34e-2} \%$ & $\num{9.11e-1} \%$  \\\hline
\end{tabular}
\end{table}

\begin{figure}
\centering 
    \includegraphics[scale=0.9]{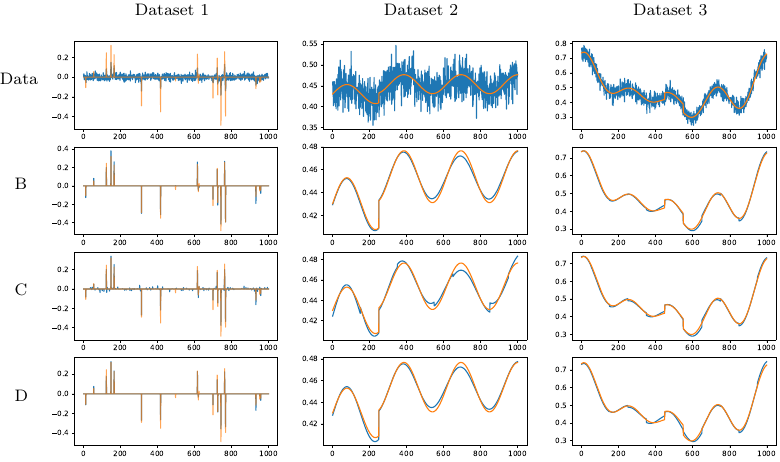}
    \caption{Qualitative comparison for the deblurring problem.
    In each column we show a signal of the test set from Datasets~\ref{GMM_10}, \ref{S1D} and \ref{FS2D}, respectively. In each row we report the original data in orange and the noisy data and the reconstructions obtained with the unsupervised \ref{item:unsup}, the dictionary learning \ref{item:DL}, and the group dictionary learning \ref{item:GDL} approach, respectively, in blue.}
\label{fig:deblurring_main_body}
\end{figure}

\section{Conclusions}
\label{sec:conclusions}
We introduced an innovative approach to sparse optimization for inverse problems, leveraging an explicit formula for the MMSE estimator in the context of a mixture of degenerate Gaussian random variables.  This methodology, rooted in statistical learning theory, follows a different approach to sparsity promotion compared to the deterministic optimization paradigm (e.g., Lasso) and the Bayesian inverse problem framework. Our reconstruction formula exhibits a notable connection to the self-attention mechanism underlying the transformer architecture, offering an efficient training process: this is useful whenever the mixture model, corresponding to the sparsity properties of the signals of interest, is unknown. This training can be conducted in both supervised and unsupervised modes.

To validate our approach, we conducted numerical implementations (both supervised and unsupervised) and compared them against established baseline algorithms for sparse optimization, such as Lasso, IHT, and their group variants. Additionally, we incorporated a dictionary learning strategy for fair comparisons. The experiments focused on three {datasets of 1D signals (discretized as vectors in $\R^n$)} featuring sparse or compressible signals, addressing denoising and deblurring tasks.

Our findings indicate that the unsupervised method consistently outperforms baseline algorithms in nearly all experiments, demonstrating superior performance with lower computational costs. However, a notable limitation of our work lies in its numerical scalability, particularly concerning larger and higher-dimensional datasets. Addressing this limitation will be a key focus in future research endeavors.

\section*{Acknowledgements}
It is a pleasure to thank Ernesto De Vito and Matti Lassas for stimulating discussions on some of the aspects of this work. This material is based upon work supported by the Air Force Office of Scientific Research under award numbers FA8655-20-1-7027 and FA8655-23-1-7083. We acknowledge the support of Fondazione Compagnia di San Paolo. Co-funded by the European Union (ERC, SAMPDE, 101041040). Views and opinions expressed are however those of the authors only and do not necessarily reflect those of the European Union or the European Research Council. Neither the European Union nor the granting authority can be held responsible for them. 
The research of LR has been funded by PNRR - M4C2 - Investimento 1.3. Partenariato Esteso PE00000013 - ``FAIR - Future Artificial Intelligence Research'' - Spoke 8 ``Pervasive AI'', which is funded by the European Commission under the NextGeneration EU programme.
The authors are members of the ``Gruppo Nazionale per l’Analisi Matematica, la Probabilità e le loro Applicazioni'', of the ``Istituto Nazionale di Alta Matematica''. The research was supported in part by the MIUR Excellence Department Project awarded to Dipartimento di Matematica, Università di Genova, CUP D33C23001110001. 

\bibliography{bibliografia}

\begin{thebibliography}{10}

\bibitem{alberti2021learning}
Giovanni~S Alberti, Ernesto De~Vito, Matti Lassas, Luca Ratti, and Matteo
  Santacesaria.
\newblock {Learning the optimal Tikhonov regularizer for inverse problems}.
\newblock {\em Advances in Neural Information Processing Systems},
  34:25205--25216, 2021.

\bibitem{beck2017first}
Amir Beck.
\newblock {\em First-order methods in optimization}.
\newblock SIAM, 2017.

\bibitem{blumensath2013compressed}
Thomas Blumensath.
\newblock Compressed sensing with nonlinear observations and related nonlinear
  optimization problems.
\newblock {\em IEEE Transactions on Information Theory}, 59(6):3466--3474,
  2013.

\bibitem{bocchinfuso2023bayesian}
A.~Bocchinfuso, D.~Calvetti, and E.~Somersalo.
\newblock Bayesian sparsity and class sparsity priors for dictionary learning
  and coding.
\newblock {\em Journal of Computational Mathematics and Data Science},
  11:100094, 2024.

\bibitem{2019-calvetti-etal}
D.~Calvetti, E.~Somersalo, and A.~Strang.
\newblock Hierachical {B}ayesian models and sparsity: {$\ell_2$}-magic.
\newblock {\em Inverse Problems}, 35(3):035003, 26, 2019.

\bibitem{Calvetti2023}
Daniela Calvetti and Erkki Somersalo.
\newblock {\em Hierarchical Models and Bayesian Sparsity}, pages 183--210.
\newblock Springer International Publishing, Cham, 2023.

\bibitem{candes2006robust}
Emmanuel~J Cand{\`e}s, Justin Romberg, and Terence Tao.
\newblock Robust uncertainty principles: Exact signal reconstruction from
  highly incomplete frequency information.
\newblock {\em IEEE Transactions on information theory}, 52(2):489--509, 2006.

\bibitem{cucker2002mathematical}
Felipe Cucker and Steve Smale.
\newblock On the mathematical foundations of learning.
\newblock {\em Bulletin of the American mathematical society}, 39(1):1--49,
  2002.

\bibitem{daubechies2004iterative}
Ingrid Daubechies, Michel Defrise, and Christine De~Mol.
\newblock An iterative thresholding algorithm for linear inverse problems with
  a sparsity constraint.
\newblock {\em Communications on Pure and Applied Mathematics: A Journal Issued
  by the Courant Institute of Mathematical Sciences}, 57(11):1413--1457, 2004.

\bibitem{de2023convergence}
Maarten~V de~Hoop, Nikola~B Kovachki, Nicholas~H Nelsen, and Andrew~M Stuart.
\newblock Convergence rates for learning linear operators from noisy data.
\newblock {\em SIAM/ASA Journal on Uncertainty Quantification}, 11(2):480--513,
  2023.

\bibitem{donoho2006compressed}
David~L Donoho.
\newblock Compressed sensing.
\newblock {\em IEEE Transactions on information theory}, 52(4):1289--1306,
  2006.

\bibitem{efron2004least}
Bradley Efron, Trevor Hastie, Iain Johnstone, and Robert Tibshirani.
\newblock Least angle regression.
\newblock {\em Annals of Statistics}, pages 407--451, 2004.

\bibitem{friedman2010note}
Jerome Friedman, Trevor Hastie, and Robert Tibshirani.
\newblock A note on the group lasso and a sparse group lasso.
\newblock {\em arXiv preprint arXiv:1001.0736}, 2010.

\bibitem{giles1987learning}
C~Lee Giles and Tom Maxwell.
\newblock Learning, invariance, and generalization in high-order neural
  networks.
\newblock {\em Applied optics}, 26(23):4972--4978, 1987.

\bibitem{gribonval2011should}
R{\'e}mi Gribonval.
\newblock Should penalized least squares regression be interpreted as maximum a
  posteriori estimation?
\newblock {\em IEEE Transactions on Signal Processing}, 59(5):2405--2410, 2011.

\bibitem{hale2007fixed}
Elaine~T Hale, Wotao Yin, and Yin Zhang.
\newblock A fixed-point continuation method for l1-regularized minimization
  with applications to compressed sensing.
\newblock {\em CAAM TR07-07, Rice University}, 43:44, 2007.

\bibitem{horesh-haber-2009}
L.~Horesh and E.~Haber.
\newblock Sensitivity computation of the {$l_1$} minimization problem and its
  application to dictionary design of ill-posed problems.
\newblock {\em Inverse Problems}, 25(9):095009, 20, 2009.

\bibitem{huang-haber-horesh-2012}
Hui Huang, Eldad Haber, and Lior Horesh.
\newblock Optimal estimation of {$\ell_1$}-regularization prior from a
  regularized empirical {B}ayesian risk standpoint.
\newblock {\em Inverse Probl. Imaging}, 6(3):447--464, 2012.

\bibitem{kovachki2023neural}
Nikola~B Kovachki, Zongyi Li, Burigede Liu, Kamyar Azizzadenesheli, Kaushik
  Bhattacharya, Andrew~M Stuart, and Anima Anandkumar.
\newblock Neural operator: Learning maps between function spaces with
  applications to pdes.
\newblock {\em J. Mach. Learn. Res.}, 24(89):1--97, 2023.

\bibitem{kundu2008gmm}
Achintya Kundu, Saikat Chatterjee, A~Sreenivasa Murthy, and TV~Sreenivas.
\newblock Gmm based bayesian approach to speech enhancement in signal/transform
  domain.
\newblock In {\em 2008 IEEE International Conference on Acoustics, Speech and
  Signal Processing}, pages 4893--4896. IEEE, 2008.

\bibitem{lee2006efficient}
Honglak Lee, Alexis Battle, Rajat Raina, and Andrew Ng.
\newblock Efficient sparse coding algorithms.
\newblock {\em Advances in neural information processing systems}, 19, 2006.

\bibitem{li2012group}
Shutao Li, Haitao Yin, and Leyuan Fang.
\newblock Group-sparse representation with dictionary learning for medical
  image denoising and fusion.
\newblock {\em IEEE Transactions on biomedical engineering}, 59(12):3450--3459,
  2012.

\bibitem{lu2012robust}
Can-Yi Lu, Hai Min, Zhong-Qiu Zhao, Lin Zhu, De-Shuang Huang, and Shuicheng
  Yan.
\newblock Robust and efficient subspace segmentation via least squares
  regression.
\newblock In {\em Computer Vision--ECCV 2012: 12th European Conference on
  Computer Vision, Florence, Italy, October 7-13, 2012, Proceedings, Part VII
  12}, pages 347--360. Springer, 2012.

\bibitem{lunz2018adversarial}
Sebastian Lunz, Ozan {\"O}ktem, and Carola-Bibiane Sch{\"o}nlieb.
\newblock Adversarial regularizers in inverse problems.
\newblock {\em Advances in neural information processing systems}, 31, 2018.

\bibitem{lustig2008compressed}
Michael Lustig, David~L Donoho, Juan~M Santos, and John~M Pauly.
\newblock {Compressed sensing MRI}.
\newblock {\em IEEE signal processing magazine}, 25(2):72--82, 2008.

\bibitem{mairal2009online}
Julien Mairal, Francis Bach, Jean Ponce, and Guillermo Sapiro.
\newblock Online dictionary learning for sparse coding.
\newblock In {\em Proceedings of the 26th annual international conference on
  machine learning}, pages 689--696, 2009.

\bibitem{mallat2008wavelet}
Stephane Mallat.
\newblock {\em A Wavelet Tour of Signal Processing: The Sparse Way}.
\newblock Academic Press, 2008.

\bibitem{mantini2021cqnn}
Pranav Mantini and Shishr~K Shah.
\newblock Cqnn: Convolutional quadratic neural networks.
\newblock In {\em 2020 25th International Conference on Pattern Recognition
  (ICPR)}, pages 9819--9826. IEEE, 2021.

\bibitem{scikit-learn}
F.~Pedregosa, G.~Varoquaux, A.~Gramfort, V.~Michel, B.~Thirion, O.~Grisel,
  M.~Blondel, P.~Prettenhofer, R.~Weiss, V.~Dubourg, J.~Vanderplas, A.~Passos,
  D.~Cournapeau, M.~Brucher, M.~Perrot, and E.~Duchesnay.
\newblock Scikit-learn: Machine learning in {P}ython.
\newblock {\em Journal of Machine Learning Research}, 12:2825--2830, 2011.

\bibitem{schwab2019deep}
Johannes Schwab, Stephan Antholzer, and Markus Haltmeier.
\newblock Deep null space learning for inverse problems: convergence analysis
  and rates.
\newblock {\em Inverse Problems}, 35(2):025008, 2019.

\bibitem{shiryaev2016probability}
Albert~N Shiryaev.
\newblock {\em Probability-1}, volume~95.
\newblock Springer, 2016.

\bibitem{tibshirani1996regression}
Robert Tibshirani.
\newblock Regression shrinkage and selection via the lasso.
\newblock {\em Journal of the Royal Statistical Society: Series B
  (Methodological)}, 58(1):267--288, 1996.

\bibitem{vaswani2017attention}
Ashish Vaswani, Noam Shazeer, Niki Parmar, Jakob Uszkoreit, Llion Jones,
  Aidan~N Gomez, {\L}ukasz Kaiser, and Illia Polosukhin.
\newblock Attention is all you need.
\newblock {\em Advances in neural information processing systems}, 30, 2017.

\bibitem{yu2011solving}
Guoshen Yu, Guillermo Sapiro, and St{\'e}phane Mallat.
\newblock Solving inverse problems with piecewise linear estimators: From
  gaussian mixture models to structured sparsity.
\newblock {\em IEEE Transactions on Image Processing}, 21(5):2481--2499, 2011.

\bibitem{yuan2006model}
Ming Yuan and Yi~Lin.
\newblock Model selection and estimation in regression with grouped variables.
\newblock {\em Journal of the Royal Statistical Society: Series B (Statistical
  Methodology)}, 68(1):49--67, 2006.

\bibitem{zelnik2012dictionary}
Lihi Zelnik-Manor, Kevin Rosenblum, and Yonina~C Eldar.
\newblock Dictionary optimization for block-sparse representations.
\newblock {\em IEEE Transactions on Signal Processing}, 60(5):2386--2395, 2012.

\end{thebibliography}
\bibliographystyle{plain}

\addresseshere

\appendix
\section*{Appendix}
\section{Proof of Theorem~\ref{thm:bayes_est}}
\label{proof_bayes_est}

The proof of Theorem~\ref{thm:bayes_est} follows easily using the next two lemmas. The first lemma provides the formula of the Bayes estimator for the denoising problem with a random variable that comes from a general mixture distribution. The second one establishes an explicit formula in the case of Gaussian mixture distributions.

\begin{lemma}
\label{denoise_mix_distr}
Suppose that Assumption~\ref{ass:0} holds true. Consider the statistical linear inverse problem \eqref{statIP}, where $x$ is sampled from a mixture of random variables, as in Definition~\ref{mix_var}, and $A \in \R^{m \times n}$. Let $Y_i=A X_i+E$, for $i=1,\dots,L$. Then
    \begin{enumerate}
        \item \label{i} the density of $Y$ is $\displaystyle p_{Y}(y) = \sum_{i=1}^L w_i p_{Y_i}(y)$, where $w_i = \mathbb{P}(I=i)$;
        \item \label{ii} $\displaystyle \E[X|Y=y] = \sum_{i=1}^L \frac{w_i p_{Y_i}(y)}{p_{Y}(y)} \E[X_i|Y_i=y]$.
    \end{enumerate}
\end{lemma}
\begin{proof} We prove the two parts separately.

{Proof of \ref{i}}. 
    The random variable $Y$ can be written as
    \begin{equation*}
     Y = A \bigg( \sum_{i=1}^L X_i \mathbbm{1}_{ \{ i \} }(I) \bigg) + E = \bigg( \sum_{i=1}^L A X_i \mathbbm{1}_{ \{ i \} }(I)\bigg) + \bigg( \sum_{i=1}^L E \mathbbm{1}_{ \{ i \} }(I) \bigg) = \sum_{i=1}^L Y_i \mathbbm{1}_{ \{ i \} }(I).  
    \end{equation*}
    Since $Y_i = A X_i+E$ and $X_i,E \perp I$ for every $i = 1,...,L$, then $Y_i \perp I$ for every $i = 1,...,L$. Therefore, the density of $Y$ becomes
    \begin{equation*}
        p_Y(y) = \sum_{i=1}^L w_i p_{Y_i}(y),
    \end{equation*}
     recalling that $w_i = \mathbb{P}(I = i)$. \\
{Proof of \ref{ii}.}
    Since $X_i$ and $E$ are independent, their joint density is $p_{X_i,E}(x,\epsilon) = p_{X_i}(x) p_E(\epsilon)$. Then, using the change of variable $\Phi(X_i,E) = (X_i, A X_i + E) = (X_i, Y_i)$, we have
    \[
    \begin{split}
        p_{X_i,Y_i}(x,y) & = p_{\Phi(X_i,E)}(\Phi(x,\epsilon)) \\ & = p_{X_i,E}(x,\epsilon) |J_{\Phi^{-1}}| \\ & = p_{X_i,E}(x,\epsilon) \\ & = p_{X_i}(x) p_E(y-A x),
    \end{split}
    \]
    since $\Phi^{-1}(X_i,Y_i) = (X_i, Y_i - A X_i)$ and 
    \[
    J_{\Phi^{-1}} = \begin{bmatrix} I & \mathbf{0} \\ -A & I \end{bmatrix}.
    \]
    The same argument holds using $X$ instead of $X_i$, then
    \begin{equation*}
    p_{X,Y}(x,y) = p_X(x) p_E(y-A x).
    \end{equation*}
    Moreover, since $X_i \perp I$ for every $i = 1,...,L$,
    $p_X(x) = \sum_{i=1}^L w_i p_{X_i}(x)$ with $w_i = \mathbb{P}(I = i)$.
    Then
    \[
    \begin{split}
        p_{X|Y}(x|y) & = \frac{p_{X,Y}(x,y)}{p_Y(y)} \\ & = \sum_{i=1}^L \frac{w_i p_{X_i}(x) p_E(y-A x)}{p_Y(y)} \\ & = \sum_{i=1}^L \frac{ w_i p_{Y_i}(y)}{p_Y(y)} \frac{p_{X_i}(x) p_E(y-A x)}{p_{Y_i}(y)} \\ & = \sum_{i=1}^L \frac{ w_i p_{Y_i}(y)}{p_Y(y)} \frac{p_{X_i,Y_i}(x,y)}{p_{Y_i}(y)} \\ & = \sum_{i=1}^L \frac{ w_i p_{Y_i}(y)}{p_Y(y)} p_{X_i|Y_i}(x|y).
    \end{split}
    \]
    Integrating in $x$, we obtain the result.
\end{proof}

\begin{lemma}
\label{denoise_mix_gauss}
Suppose that Assumption~\ref{ass:0} holds true. Consider the statistical linear inverse problem \eqref{statIP}, where $x$ is sampled from a Gaussian mixture, the noise $\epsilon$ is sampled from $E\sim \mathcal{N}(\mathbf{0},\Sigma_E)$, independent of $X_i$ and $I$, and $A \in \R^{m \times n}$. Let $Y=A X+E$. Then
    \begin{enumerate}
        \item \label{1} the density of $Y_i$ is 
        \[
        p_{Y_i}(y) = \frac{1}{\sqrt{(2\pi)^n |A \Sigma_i A^T + \Sigma_E|}} \exp{ \Big( -\frac{1}{2} \| (A \Sigma_i A^T + \Sigma_E)^{-\frac{1}{2}} (y-A \mu_i) \|_2^2 \Big) };
        \]
        \item \label{2} and $\E[X_i|Y_i=y] = \mu_i + \Sigma_i A^T (A \Sigma_i A^T + \Sigma_E)^{-1} (y-A \mu_i)$.
    \end{enumerate}
\end{lemma}
\begin{proof} We prove the two parts separately.

{Proof of \ref{1}}:
Since $X_i \perp E$, $X_i \sim \mathcal{N}(\mu_i,\Sigma_i)$ and $E \sim \mathcal{N}(\mathbf{0},\Sigma_E)$, we have
\[
Y_i = A X_i + E \sim \mathcal{N}(A \mu_i,A \Sigma_i A^T+\Sigma_E).
\]
The expression for the density of $Y_i$ immediately follows.

{Proof of \ref{2}}:
Since $Y_i = A X_i + E$ and $X_i \perp E$, we have
\[
\begin{split}
    (X_i,Y_i) & \sim \mathcal{N} \Bigg( \begin{bmatrix} \E[X_i] \\ \E[Y_i] \end{bmatrix} , \begin{bmatrix} \var[X_i] & \cov[X_i,Y_i] \\ \cov[Y_i,X_i] & \var[Y_i] \end{bmatrix} \Bigg)
    \\ & = \mathcal{N} \Bigg( \begin{bmatrix} \mu_i \\ A \mu_i \end{bmatrix} , \begin{bmatrix} \Sigma_i & \Sigma_i A^T \\ A \Sigma_i & A \Sigma_i A^T + \Sigma_E \end{bmatrix} \Bigg).
\end{split}
\]
Then, using \cite[Theorem $2$, Section $13$]{shiryaev2016probability}, the conditional distribution $(X_i|Y_i = y)$ is Gaussian and its expectation is given by
\begin{equation*}
    \E[X_i|Y_i=y] = \mu_i + \Sigma_i A^T (A \Sigma_i A^T+ \Sigma_E)^{-1} (y-A \mu_i).\qedhere
\end{equation*}
\end{proof}

\section{Implementation of the baseline algorithms}
\label{app:algorithms}

We collect here some additional information regarding the baseline techniques presented in Section \ref{sec:comp}. In particular, we specify the numerical algorithms we employ to minimize the involved functionals and provide further implementation details.

\subsection{LASSO}

The minimization of the LASSO functional \eqref{LASSO_func}, introduced in  Section \ref{sec:LASSO}, can be performed employing the Iterative Soft Thresholding Algorithm (ISTA) \cite{daubechies2004iterative,hale2007fixed}, which is a proximal-gradient descent method involving the computation of the proximal operator for the convex and non-smooth term of the functional (the $\ell^1$-norm), and the gradient descent step for the smooth term (the data fidelity). The proximal operator of the $\ell^1$-norm is the soft thresholding operator, from which ISTA takes its name. 
Mathematically, the ($k+1$)-th iteration is
\begin{equation}
\label{ISTA_iter}
    \beta^{k+1} = S_{t \lambda} (\beta^k - t M^T A^T (A M \beta^k - y) ),
\end{equation}
where $t > 0$ is a stepsize and $S_\lambda \colon \R^n \to \R^n$ is the soft thresholding operator defined componentwise as 
\begin{equation*}
   S_{\lambda}(\beta)_i = \max \{|\beta_i|-\lambda,0 \} \hspace{0.1cm} \sign(\beta_i),
\end{equation*}
where $i$ indicates the component.

For the denoising problem ($A = I$), setting $t=1$, algorithm \eqref{ISTA_iter} achieves convergence in a single step. Therefore, the solution $\bar{\beta}$ is 
\begin{equation}
\label{ISTA_denoising}
   \bar{\beta} = S_\lambda (M^T A^T y).
\end{equation}

\subsection{Group LASSO}

The minimization of \eqref{GLASSO_func_gen} can be performed via the iterative algorithm proposed in \cite[Proposition $1$]{yuan2006model}, showing strong connections with the proximal gradient descent method, which reads
\begin{equation*}
\beta^{k+1} = \prox_{t \lambda f} (\beta^k - t \tilde{M}^T A^T (A \tilde{M} \beta^k - y) ),
\end{equation*}
where $f(\beta) = \sum_{i=1}^L  \| \beta_i \|_{K_i}$, $\tilde{M} = [ M_1 | M_2 | ... | M_L ]$, and $t > 0$ is a stepsize. The map $\prox_{t\lambda f}$ satisfies (see \cite[Theorem $6.6$, Chapter $6$]{beck2017first})
\begin{equation*}
    \prox_{t \lambda f} (\beta)= (\prox_{t \lambda f_i}(\beta_i))_{i=1}^L,
\end{equation*}
being $f_i(\beta_i) = \| \beta_i \|_{K_i}$, and the proximal operator of the $\ell^2$ weighted norm is (see \cite[Lemma $6.68$, Chapter $6$]{beck2017first})
\begin{equation*}
 \prox_{t \lambda f_i}(\beta_i) = \begin{cases}
     \beta_i - K_i^T (K_i K_i^T)^{-1} K_i \beta_i & \quad \| (K_i K_i^T)^{-1} K_i \beta_i \|_2 \leq t \lambda \\
     \beta_i - K_i^T (K_i K_i^T + \alpha^* I)^{-1} K_i \beta_i & \quad \| (K_i K_i^T)^{-1} K_i \beta_i \|_2 > t \lambda
 \end{cases},   
\end{equation*}
where $\alpha^*$ is the unique positive root of the non decreasing function $g(\alpha) := \| (K_i K_i^T +  \alpha I)^{-1} K_i \beta_i \|_2^2 - (t \lambda)^2$.

The resulting algorithm is computationally rather expensive: indeed, the iterative computation of the proximal operator, involving the root-finding problem, must be applied on each vector $y_j$ separately.
To ease the computational burden, in our experiments we instead minimize \eqref{GLASSO_func_gen} through the ADAM optimization scheme implemented in the \textit{pytorch} library, which can process multiple signals in parallel.

\subsection{IHT}

The iterative algorithm to minimize \eqref{IHT_func} consists in the alternation of a gradient descent step in the direction given by the MSE and a projection onto $\mathcal{S}$, namely the ($k+1$)-th iteration is
\begin{equation}
\label{IHT_iter}
\beta^{k+1} = P_{\mathcal{S}}(\beta^k-t M^TA^T(AM\beta^k - y)),
\end{equation}
where $P_{\mathcal{S}}(\beta)=P_{\mathcal{S}_{\bar i}}(\beta)$, $\bar{i} = \argmin_{i} \{ \| P_{\mathcal{S}_i}(\beta) - x \|_2 \}$, $P_{\mathcal{S}_i}$ is the orthogonal projection onto $\mathcal{S}_i$ and $t > 0$ is a stepsize. The projection onto coordinate hyperplanes is performed by simply setting to $0$ all the non-active coordinates. The overall projection $P_{\mathcal{S}}(\beta)$ is also very efficient, since it consists only of computing $L$ projections onto the different subspaces $\mathcal{S}_i$ and selecting the optimal one. Moreover, if $\mathcal{S}$ is the union of all coordinate hyperplanes of dimension $s$, $P_{\mathcal{S}}(\beta)$ simply reduces to selecting the $s$ largest components of $\beta$.

For the denoising problem ($A = I$) the IHT algorithm \eqref{IHT_iter} with $t = 1$ converges to the solution in a single step. Therefore, using the same definitions as before, the solution is 
\begin{equation*}
    \bar{\beta} = P_{\mathcal{S}_{\bar{i}}}(M^T y).
\end{equation*}

\subsection{Dictionary Learning}

Once a dictionary has been learned through the technique discussed in Section \ref{sec:dict_learn}, we are left with the minimization of the functional $\mathcal{F}$ as in \eqref{DL_func}. We do so by means of a proximal-gradient scheme, whose iterations read as
\begin{equation}
\label{dict_learn_iter}
x^{k+1} = \operatorname{prox}_{\lambda t G}(x^k-t A^T(A x^k - y)),
\end{equation}
where $t > 0$ is a stepsize and  
\[
\begin{aligned}
\operatorname{prox}_{G}(z) &= \argmin_{x \in \operatorname{Im}(D)}\left\{ 
\frac{1}{2}\|x-z\|^2 + \|D^{+} x\|_1 \right\} \\
& = D \argmin_{\beta \in \R^d}\left\{ 
\frac{1}{2}\|D\beta-z\|^2 + \|\beta\|_1 \right\}.    
\end{aligned}
\]

For the denoising problem ($A = I$) the algorithm \eqref{dict_learn_iter} with $t = 1$ converges to the solution $\bar{x}$ in a single step, namely 
\begin{equation*}
    \bar{x} = D \operatorname{prox}_{\lambda G}(y).
\end{equation*}

Notice that, under the injectivity assumption on the matrix $D$, the minimization of $\mathcal{F}$ in \eqref{DL_func} is equivalent to a LASSO problem in synthesis formulation, using $D$ as a synthesis operator. Despite the minimizers of those two problems are the same, the iterates approximating them by proximal-gradient schemes do not coincide. For numerical reasons, we prefer the formulation \eqref{DL_func}, leading to the iterations \eqref{dict_learn_iter}. 
Indeed, the implementation of $\operatorname{prox}_G$ can be done efficiently by means of the same routines employed for dictionary learning. Specifically, we rely on the routines provided in the \textit{scikit-learn} library \cite{scikit-learn}, and compute $\operatorname{prox}_G$ through the \textit{transform} method of the learned dictionary.

Finally, in the case of Group Dictionary Learning \eqref{eq:GDL}, we minimize $\mathcal{F}$ by means of a proximal-gradient scheme as in \eqref{dict_learn_iter}, replacing $\prox_{G}$ by $\prox_{\hat{G}}$, which can be easily computed as\[
\prox_{\hat{G}}(z) = \prox_{G_I}(z),\quad I = \argmin_{i = 1,\ldots,L}\left\{ \frac{1}{2}\| z- \prox_{G_i}(z)\|^2 + G_i\big(\prox_{G_i}(z)\big) \right\}.
\]

\section{Additional numerical results}
\label{app:numerics}

In Figure~\ref{fig:denoising_appendix}, we show a qualitative comparison for the denoising problem described in Section~\ref{sec:denois_comp} between the reconstructions obtained with {our supervised method (with random initialization),} our unsupervised method, IHT with SVD basis, IHT with SVD bases of groups, IHT with known basis, LASSO with SVD basis, Group LASSO with SVD bases and LASSO with known basis.

\begin{figure}
\centering 
    \includegraphics[scale=0.9]{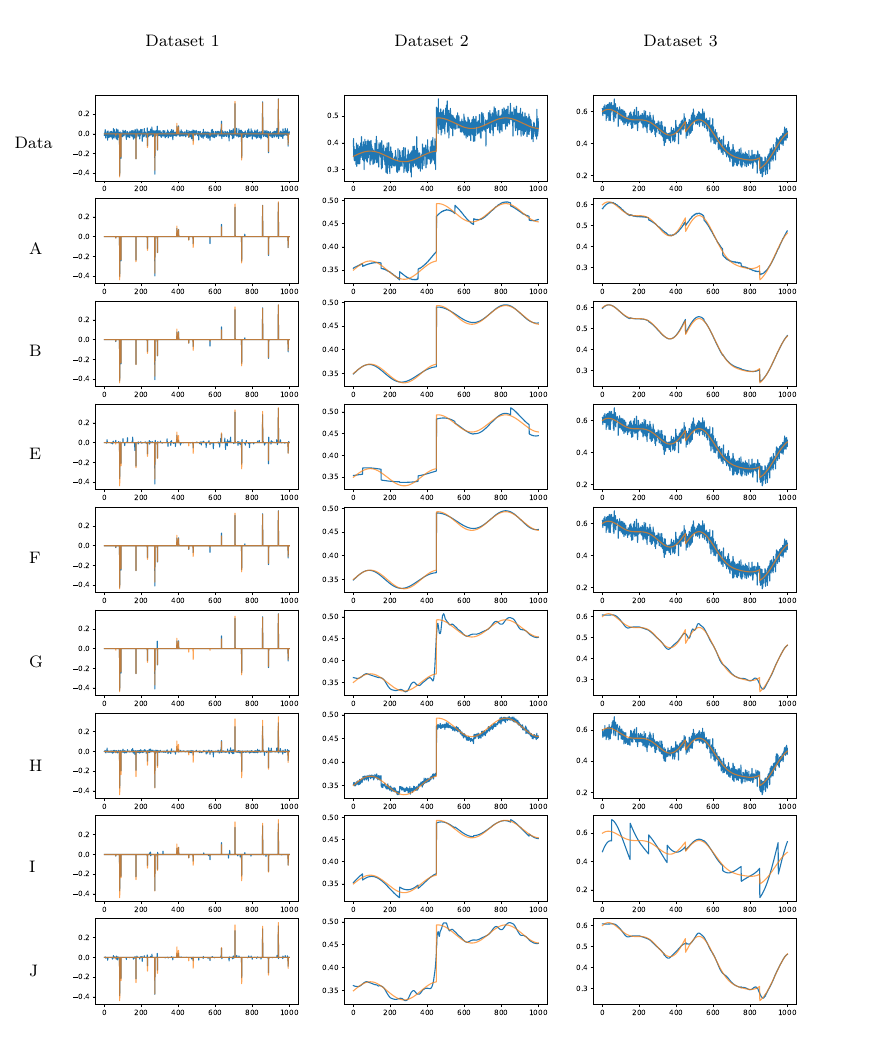}
    \caption{Further qualitative comparisons for the denoising problem.
    In each column we show a signal of the test set from Datasets~\ref{GMM_10}, \ref{S1D} and \ref{FS2D}, respectively. In each row we report the original data in orange and the noisy data, and the reconstructions obtained with the {randomly-initialized supervised \ref{item:sup},} unsupervised \ref{item:unsup}, IHT with SVD basis \ref{item:IHT_SVD}, IHT with SVD bases of groups \ref{item:GIHT_SVD}, IHT with known basis  \ref{item:IHT_known}, LASSO with SVD basis \ref{item:LASSO_SVD}, Group LASSO with SVD bases \ref{item:GLASSO_SVD}, and LASSO with known basis \ref{item:LASSO_known} approach, respectively, in blue.}
\label{fig:denoising_appendix}
\end{figure}

In Figure~\ref{fig:deblurring_appendix}, we show a qualitative comparison for the deblurring problem described in Section~\ref{sec:deblur_comp} between the reconstructions obtained with our unsupervised method, IHT with SVD bases of groups and Group LASSO with SVD bases.

\begin{figure}
\centering 
    \includegraphics[scale=0.9]{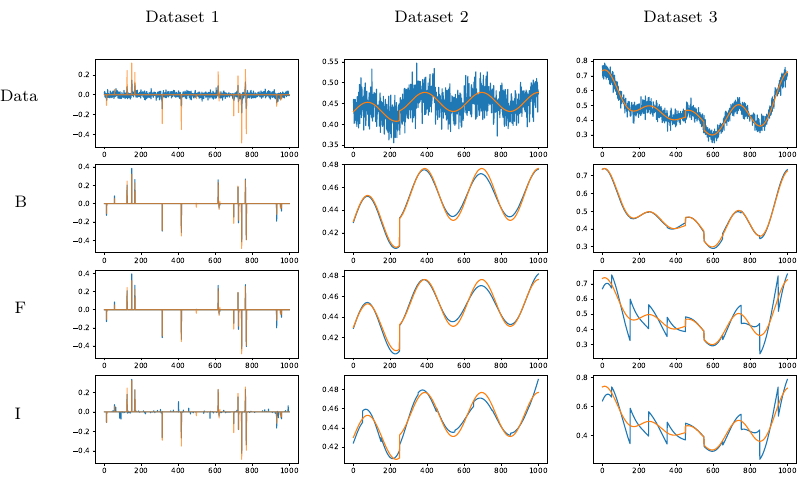}
    \caption{Further qualitative comparisons for the deblurring problem.
    In each column we show a signal of the test set from Datasets~\ref{GMM_10}, \ref{S1D} and \ref{FS2D}, respectively. In each row, we report the original data in orange and the noisy data, and the reconstructions obtained with the unsupervised \ref{item:unsup}, IHT with SVD bases of groups \ref{item:GIHT_SVD}, and Group LASSO with SVD bases \ref{item:GLASSO_SVD} approach, respectively, in blue.}
\label{fig:deblurring_appendix}
\end{figure}

\end{document}